\documentclass[11pt]{article}
\usepackage{rotating,amsmath,amssymb,amsfonts,amsthm}
\usepackage{epsfig,mathrsfs,natbib,color}
\usepackage{graphicx,graphics,multirow}
\usepackage{array,wrapfig}
\usepackage{hyperref}
\usepackage{booktabs,adjustbox,arydshln}
\usepackage{subcaption}
\usepackage[utf8]{inputenc}  
\usepackage[T1]{fontenc}     
\usepackage{url}             
\usepackage{nicefrac}   
\usepackage{microtype}       
\usepackage[ruled,vlined]{algorithm2e}
\def\boxit#1{\vbox{\hrule\hbox{\vrule\kern6pt \vbox{\kern6pt#1\kern5pt} \kern6pt\vrule}\hrule}}

\newcommand{\tpm}{$\pm$}

\newcommand{\bX}{{\boldsymbol X}}
\newcommand{\bY}{{\boldsymbol Y}}

\newcommand{\by}{{\boldsymbol y}}
\newcommand{\bx}{{\boldsymbol x}} 
\newcommand{\bz}{{\boldsymbol z}}
\newcommand{\bZ}{{\boldsymbol Z}}
\newcommand{\bb}{{\boldsymbol b}}
\newcommand{\be}{{\boldsymbol e}}
\newcommand{\bw}{{\boldsymbol w}}
\newcommand{\bv}{{\boldsymbol v}}

\newcommand{\ba}{{\boldsymbol a}}

\newcommand{\bk}{{\boldsymbol k}}
\newcommand{\bK}{{\boldsymbol K}}

\newcommand{\bI}{{\boldsymbol I}}

\newcommand{\bs}{{\boldsymbol s}}

\newcommand{\mR}{\mathbb{R}}

\newcommand{\BY}{\mathbb{Y}}

\newcommand{\mbE}{{\mathbb{E}}}

\newcommand{\bbeta}{{\boldsymbol \beta}}

\newcommand{\btheta}{{\boldsymbol \theta}}

\newcommand{\bvartheta}{{\boldsymbol \vartheta}}
\newcommand{\bepsilon}{{\boldsymbol \epsilon}}

\newcommand{\bmu}{{\boldsymbol \mu}}

\newcommand{\var}{{\mbox{Var}}}

\newcommand{\diag}{{\mbox{diag}}}
\newcommand{\bYmis}{{\boldsymbol Y}_{{\rm mis}}}

\newcommand{\bXmis}{{\boldsymbol Z}_{{\rm mis}}}
\newcommand{\bXobs}{{\boldsymbol Z}_{{\rm obs}}}

\newtheorem{theorem}{Theorem}[]
\newtheorem{lemma}{Lemma}[]

\newtheorem{assump}{Assumption}

\begin{document}

\title{A Kernel-Expanded Stochastic Neural Network}

 \author{Yan Sun and Faming Liang\thanks{To whom correspondence should be addressed: Faming Liang.
  F. Liang is Professor (email: fmliang@purdue.edu),  
  Y. Sun is Graduate Student, Department of Statistics,
  Purdue University, West Lafayette, IN 47907.
 } 
 }
 
 \date{} 
 
 \maketitle 
 
\begin{abstract}
The deep neural network suffers from many fundamental issues in machine learning. For example, it often gets trapped into a local minimum in training, and its prediction uncertainty is hard to be assessed. To  address these issues, we propose the so-called kernel-expanded stochastic neural network (K-StoNet) model, which incorporates support vector regression (SVR) as the first hidden layer and reformulates the neural network as a latent variable model. The former maps the input vector into an infinite dimensional feature space via a radial basis function (RBF) kernel, ensuring absence of local minima on its training loss surface. The latter breaks the high-dimensional nonconvex neural network training problem into a series of low-dimensional convex optimization problems, and enables its prediction uncertainty easily assessed. The K-StoNet can be easily trained using the imputation-regularized optimization (IRO) algorithm. Compared to traditional deep neural networks,  K-StoNet possesses a theoretical guarantee to asymptotically converge to the global optimum and enables the prediction uncertainty easily assessed. The performances of the new model in training, prediction and uncertainty quantification are illustrated by simulated and real data examples. 

\vspace{5mm}


\end{abstract}

\section{Introduction}

Deep learning has been the engine powering many of the recent successes of artificial intelligence (AI). However, the deep neural network (DNN), as 
the basic model of deep learning, still suffers from many fundamental issues from the perspective of statistical modeling. For example, it often gets trapped into a local energy minimum and its prediction uncertainty is hard to be quantified. In consequence, it is often unclear 
whether a DNN is guaranteed to 
have a desired property after training instead of getting trapped into an arbitrarily poor local minimum, and whether its decision/prediction is reliable. These issues make the trustworthiness of AI highly questionable.

Toward understanding the optimization process of the DNN training, a line of researches have been 
done from the perspective of over-parameterization. 
For example, \cite{GoriTesi1992} and \cite{Nguyen2017TheLS} studied the training loss surface of over-parameterized DNNs. They showed that for a fully connected DNN, almost all local minima are globally optimal, if the width of one
layer of the DNN is no smaller than the training sample size and the network structure from
this layer on is pyramidal. Recently, \cite{AllenZhu2019ACT,Du2019GradientDF, Zou2020GradientDO,ZouGu2019} explored the convergence theory of the gradient-based algorithms in training over-parameterized DNNs. They showed that the gradient-based algorithms with random initialization can converge to global minima provided that the 
width of the DNN is polynomial in training sample size. With some training tricks, such as early stopping and batch normalization \cite{batchnormalization2015}, the over-parameterized DNNs may work well in prediction \cite{Doubledescent2020}, but are no longer well calibrated as shown in \cite{CalibrationDNN2017}; that is,
 their prediction probabilities do not effectively reflect the true likelihood of the events \cite{KuhnJ2013}.
Motivated by this observation, 
 uncertainty quantification for deep learning has received much attention in the recent literature. Various methods of uncertainty quantification have been proposed, e.g., Bayesian sparse deep learning \cite{Wang2020UncertaintyQF}, Monte Carlo dropout
 \cite{MCdropout2016}, and deep ensemble \cite{Deepensemble2017}, which all work with an  ensemble of DNNs. However,
 the methods based on a single DNN  are rare.

This paper proposes a new neural network model, the so-called kernel-expanded stochastic neural network (or K-StoNet in short), which overcomes the issues on local trap and uncertainty quantification suffered by the DNN in a coherent way. The new model incorporates support vector regression (SVR) \citep{VapnikL1963,VapnikC1964}
as the first hidden layer and reformulates the neural network as a latent variable model. The former maps the input vector from its original space into an infinite dimensional feature space, ensuring all local minima on the loss surface are globally optimal. 
The latter resolves the parameter optimization 
and statistical inference issues associated with the neural network: it breaks the high-dimensional nonconvex neural network training problem into a series of low-dimensional convex optimization problems, and enables the prediction uncertainty easily assessed.
The new model can be easily trained using the imputation-regularized optimization (IRO) algorithm \cite{Liang2018missing}, which converges very fast, usually within a small number of epochs. Moreover, the introduction of the SVR layer with a universal kernel \citep{UnivKernel2006,Hammer2003} enables K-StoNet to work with a smaller network, while ensuring the universal approximation capability.  
{\it In summary, this work provides a new neural network model which possesses a theoretical guarantee to asymptotically converge to the global optimum and enables the prediction uncertainty easily assessed.}
As discussed in Section \ref{discussionsection}, we expect that this work will have many implications toward the development of trustworthy AI.

As shown later, reformulating the neural network as a latent variable model (or called a stochastic neural network) is crucial for training the kernel-expanded neural network and making valid inference for its prediction uncertainty. 
We note that 
stochastic neural networks have a long history in machine learning.  Famous 
examples include deep belief networks \citep{Hinton2007} and deep Boltzmann machines \citep{SHinton2009}, which have ever advanced the development of machine learning. Recently, some researchers have proposed to add noise to the DNN to improve its performance.
For example, \cite{srivastava2014dropout} proposed the dropout method to prevent the DNN from over-fitting by randomly dropping some hidden and visible units during training; 
\cite{Neelakantan2017AddingGN} proposed to add gradient noise to improve training; and \cite{You2018AdversarialNL, Noh2017RegularizingDN, Glehre2016NoisyAF} proposed to 
use stochastic activations through adding noise to improve 
generalization and adversarial robustness. 
However, these methods are usually not systematic and theoretical guarantees are hard to be provided. In contrast, K-StoNet is developed under a rigorous statistical framework, whose convergence to the global optimum is asymptotically guaranteed and whose prediction uncertainty can be easily assessed.

The remaining part of this paper is organized as follows. Section \ref{modelsection} describes the K-StoNet model and the IRO algorithm. Sections \ref{numericalI} and \ref{numericalII} illustrate the performance of K-StoNet using simulated and real data problems. Section \ref{uncertainsection} describes how to quantify prediction uncertainty for K-StoNet. Section \ref{discussionsection} concludes the  paper with a brief discussion.
 
\section{A Kernel-Expanded Stochastic Neural Network} 
\label{modelsection}

\subsection{A Kernel-Expanded Neural Network}

Let's start with a brief review for the theory developed in \cite{GoriTesi1992} and \cite{Nguyen2017TheLS}. 
Consider a neural network model with $h$ hidden layers. 
Let $\bZ_0=\bX \in \mR^{m_0}$ denote an input vector, let $\bZ_i \in \mR^{m_i}$ denote the output vector at layer $i$ for $i=1,2,\ldots,h,h+1$, and let $\bY \in \mR^{m_{h+1}}$ denote the target output. At each  layer $i$, the neural network calculates its output: 
\begin{equation} \label{DNNoperator}
\bZ_i=\Psi(\bw_i \bZ_{i-1}+\bb_i),  \quad i=1,2,\ldots,h, h+1, 
\end{equation}
where $\bw_i \in \mR^{m_i}\times \mR^{m_{i-1}}$ and $\bb_i \in \mR^{m_i}$ denote the weights and bias of the layer $i$ respectively,  $\Psi(\bs)=(\psi(s_1),\ldots,\psi(s_{m_i}))^T$, and 
$\psi(\cdot)$ is the activation function used in the network. 
For convenience, let $p=m_0$ denote the dimension of the input vector, let  $\tilde{\bw}_i=[\bw_i,\bb_i]$ denote the matrix of all parameters of layer $i$ for $i=1,2,\ldots,h+1$, and let $\btheta=(\tilde{\bw}_1,\tilde{\bw}_2,\ldots,\tilde{\bw}_{h+1})\in \Theta$.  
Further, we assume that the network structure is pyramidal with 
$m_0 \geq m_1 \geq \cdots \geq m_h \geq m_{h+1}$ and, for simplicity, the same activation function $\psi(\cdot)$ is used for all hidden units. 
Let $U: \Theta\to \mathbb{R}$ be the loss function of the neural network, which is given by
\[
 U(\btheta)=-\frac{1}{n} \sum_{i=1}^n \log \pi(\bY^{(i)}|\btheta,\bX^{(i)}) 
 \stackrel{\Delta}{=} \frac{1}{n} \sum_{i=1}^n  l(\bZ_{h+1}^{(i)}),
\]
 where $\pi(\cdot)$ denotes the density/mass function of each observation under the neural network model, $n$ denotes the training sample size, $i$ indexes the training sample, and ${\bZ}_{h+1}^{(i)}$ is the output vector of layer $h+1$, and $l: \mathbb{R}^{m_{h+1}}\to \mathbb{R}$ is assumed to be a continuously differentiable loss function, i.e., $l\in C^2(\mathbb{R}^{m_{h+1}})$. 
 In order to study the property of the loss function, \cite{Nguyen2017TheLS} made the following assumption: 

\begin{assump}  \label{ass1}
\begin{itemize}
    \item[(i)] All training samples are distinct, i.e., $\bX^{(i)} \ne \bX^{(j)}$ for all $i \ne j$;
    \item[(ii)] $\psi(\cdot)$ is real analytic, strictly monotonically increasing and (a) $\psi(\cdot)$ is bounded or (b) there are positive $\rho_1$, $\rho_2$, $\rho_3$ and $\rho_4$ such that $|\psi(t)|\leq \rho_1 e^{\rho_2 t}$ for $t<0$ and $|\psi(t)|\leq \rho_3 t+\rho_4$ for $t\geq 0$;
    \item[(iii)] $l\in C^2(\mathbb{R}^{m_{h+1}})$ and 
    if $l'(\ba)=0$ then $\ba$ is a global optimum. 
\end{itemize}
\end{assump}

Here a function $\psi: \mathbb{R} \to \mathbb{R}$ 
 is called real analytic if the corresponding Taylor series converges to $\psi(s)$ on an open subset of $\mathbb{R}$. It is easy to see that many of the activation functions, such as
 {\it tanh}, {\it sigmoid} and {\it softplus}, satisfy \ref{ass1}-(ii). It is known that the {\it softplus} function can be viewed as a differentiable approximation to ReLU. \ref{ass1}-(iii) can be satisfied by any twice continuously differentiable convex loss function, e.g., negative log-Gaussian and log-binomial density/mass functions. 
The following lemma is a restatement of Theorem 3.4 of \cite{Nguyen2017TheLS}. A similar result has also been established in \cite{GoriTesi1992}. 

\begin{lemma} (Theorem 3.4 of \cite{Nguyen2017TheLS}) \label{lem1}   Suppose Assumption \ref{ass1} holds. If (i) the training samples are linearly independent, i.e.,
 rank$([\bX,1_n])=n$; and (ii) the weight matrices $(\tilde{\bw}_l)_{l=2}^{h+1}$ have full row rank, i.e., rank$(\tilde{\bw}_l)=m_l$ for 
 $l=2,3,\ldots,h+1$, then every critical point of the loss function $U(\btheta)$ is a global minimum. 
 \end{lemma}

Among the conditions of Lemma \ref{lem1}, Assumption \ref{ass1} is regular as discussed above,  and condition (ii) can be almost surely satisfied by restricting the network structure to be pyramidal. However, condition (i) is not satisfied by many machine learning problems for which the training sample size is much larger than the dimension of the input. To have this condition satisfied, we propose a kernel-expanded neural network (or KNN in short), where each input vector
 $\bx$ is mapped into an infinite dimensional feature space by a radial basis function (RBF) kernel $\phi(\bx)$. 
 More precisely, the KNN can be expressed as 
\begin{equation} \label{KDNNeq}
\begin{split}
\tilde{\bY}_1 & =\bb_1+\bbeta \phi(\bX), \\
\tilde{\bY}_i & =\bb_i+\bw_i \Psi(\tilde{\bY}_{i-1}), \quad i=2,3,\ldots,h, \\
\bY & =\bb_{h+1}+\bw_{h+1} \Psi(\tilde{\bY}_h)+\be_{h+1},
\end{split}
\end{equation}
where $\be_{h+1} \sim N(0,\sigma_{h+1}^2I_{m_{h+1}})$ is Gaussian random error; $\tilde{\bY}_i, \bb_i \in \mathbb{R}^{m_i}$ for $i=1,2,\ldots,h$;  $\bY_{h+1},\bb_{h+1} \in \mathbb{R}^{m_{h+1}}$;
$\Psi(\tilde{\bY}_{i-1})=(\psi(\tilde{Y}_{i-1,1}), \psi(\tilde{Y}_{i-1,2}), \ldots,\psi(\tilde{Y}_{i-1,m_{i-1}}))^T$ for $i=2,3,\ldots,h+1$, 
and $\tilde{Y}_{i-1,j}$ is the $j$th element of $\tilde{\bY}_{i-1}$; $\bw_i \in \mathbb{R}^{m_i \times m_{i-1}}$ for $i=2,3,\ldots, h+1$; $\bbeta \in \mathbb{R}^{m_1 \times d_{\phi}}$ and $d_{\phi}$ denotes the dimension of the feature space  of the kernel $\phi(\cdot)$. For the RBF kernel, $d_{\phi}=\infty$.
Note that different kernels can be used for different hidden units of the first hidden layer. For notational simplicity, we consider only  
the case that  the same kernel is used for all the hidden units and $\bY$ follows a normal regression model.  Replacing the third equation of (\ref{KDNNeq}) by a logit model will lead to the classification case. 
In general, we consider only 
the distribution $\pi(\bY|\btheta,\bX)$ such that 
Assumption \ref{ass1}-(iii) is satisfied, where, with a slight abuse of notation, we use $\btheta=(\bb_1,\bbeta; \bb_2,\bw_2; \ldots,\bb_{h+1},\bw_{h+1})$ to denote the collection of all weights of the KNN. 

Compared to formula (\ref{DNNoperator}),  
formula (\ref{KDNNeq}) gives a new presentation 
form for neural networks, where the feeding operator (used for calculating $\bw_i \bZ_{i-1}+\bb_i$) and the activation operator $\psi(\cdot)$ are separated into two equations. As shown later, such a representation facilitates parameter estimation for the neural network when auxiliary noise are introduced into the model.

For KNN, since the input vector has been mapped into an infinite dimensional feature space,  the Gram matrix $\bK=(k_{ij})$, where $k_{ij}=\phi^T(\bx_i) \phi(\bx_j)$, is of full rank, i.e., rank$(\bK)=n$. 
This means the transformed samples $\phi(\bX^{(1)}), \phi(\bX^{(2)}), \ldots, \phi(\bX^{(n)})$ 
are linearly independent. In addition, we can restrict 
 the structure of the KNN to be pyramidal, and choose the activation and loss function such that Assumption \ref{ass1} is satisfied. Therefore, 
by Lemma \ref{lem1}, every critical point of the KNN model is a global minimum. In summary, we have the following theorem with the proof as argued above. 

\begin{theorem}\label{thm:1}
For a KNN model given in (\ref{KDNNeq}), if Assumption \ref{ass1} holds, an RBF kernel is used in the input layer, and the weight matrices $(\tilde{\bw}_l)_{l=2}^{h+1}$ are of 
full row rank, then every critical point of its loss function is a global minimum. 
\end{theorem} 

Other than the RBF kernel, the polynomial kernel might
also satisfy Theorem \ref{thm:1} for certain problems. For an input vector $\bx\in \mathbb{R}^p$, the dimension of its feature space is $\binom{p+q}{q}$, where $q$ denotes the degree freedom of the polynomial kernel. Therefore, if the resulting Gram matrix is of full rank, then the transformed samples $\phi(\bX^{(1)}), \ldots, \phi(\bX^{(n)})$ are also linearly independent.
However, as stated in Assumption \ref{ass4}, the K-StoNet requires the kernel to be universal, so the polynomial kernel is not used in this paper.

\subsection{A Kernel-Expanded StoNet as an Approximator to KNN} \label{ksection}

As shown in Theorem \ref{thm:1}, the KNN has a  nice loss surface, where every critical point is a global minimum.  However, training the KNN using a gradient-based algorithm is infeasible, as the transformed features are not explicitly available. Based on the kernel representer theorem \citep{Wahba1990, Scholkopf2001}, one might consider to \textcolor{black}{replace} the first equation of (\ref{KDNNeq}) by 
\begin{equation}
\label{KNN_kernel_form}
\tilde{\bY}_1=\bb_{1}+\sum_{i=1}^n 
 \bw_{1}^{(i)} K(\bX^{(i)},\bX),
\end{equation}
where $\bw_1^{(i)} \in  \mathbb{R}^{m_1}$ and 
$K(\bX^{(i)},\bX)=\phi^T(\bX^{(i)}) \phi(\bX)$ is explicitly available, and then train such an over-parameterized neural network model using a regularization method. 
\textcolor{black}{However, the global optimality property established in Theorem \ref{thm:1} might not hold for the regularized KNN any more, because 
the proof of Theorem \ref{thm:1} relies on the back propagation formula of the neural network (see the proof of Theorem 3.4 in \cite{Nguyen2017TheLS} for the detail), while that formula cannot be easily generalized to regularized loss functions.}
\textcolor{black}{Moreover, for a nonlinear kernel regression $\bY=g(\tilde{\bY}_1)+\be=g(\bb_1 + \bbeta \phi(\bX))+\be$, where $g(\cdot)$ represents a nonlinear mapping from 
 $\tilde{\bY}_1$ to the output layer, the kernel representer theorem does not hold for $g(\cdot)$ in general and, therefore, (\ref{KNN_kernel_form}) and the first equation of (\ref{KDNNeq}) might not be equivalent for the KNN. Recall that SVR is a special case of the kernel regression with the identity mapping $g(\tilde{\bY}_1)=\tilde{\bY}_1$.}

\begin{figure}[htbp]
\centering
\includegraphics[width=0.75\linewidth]{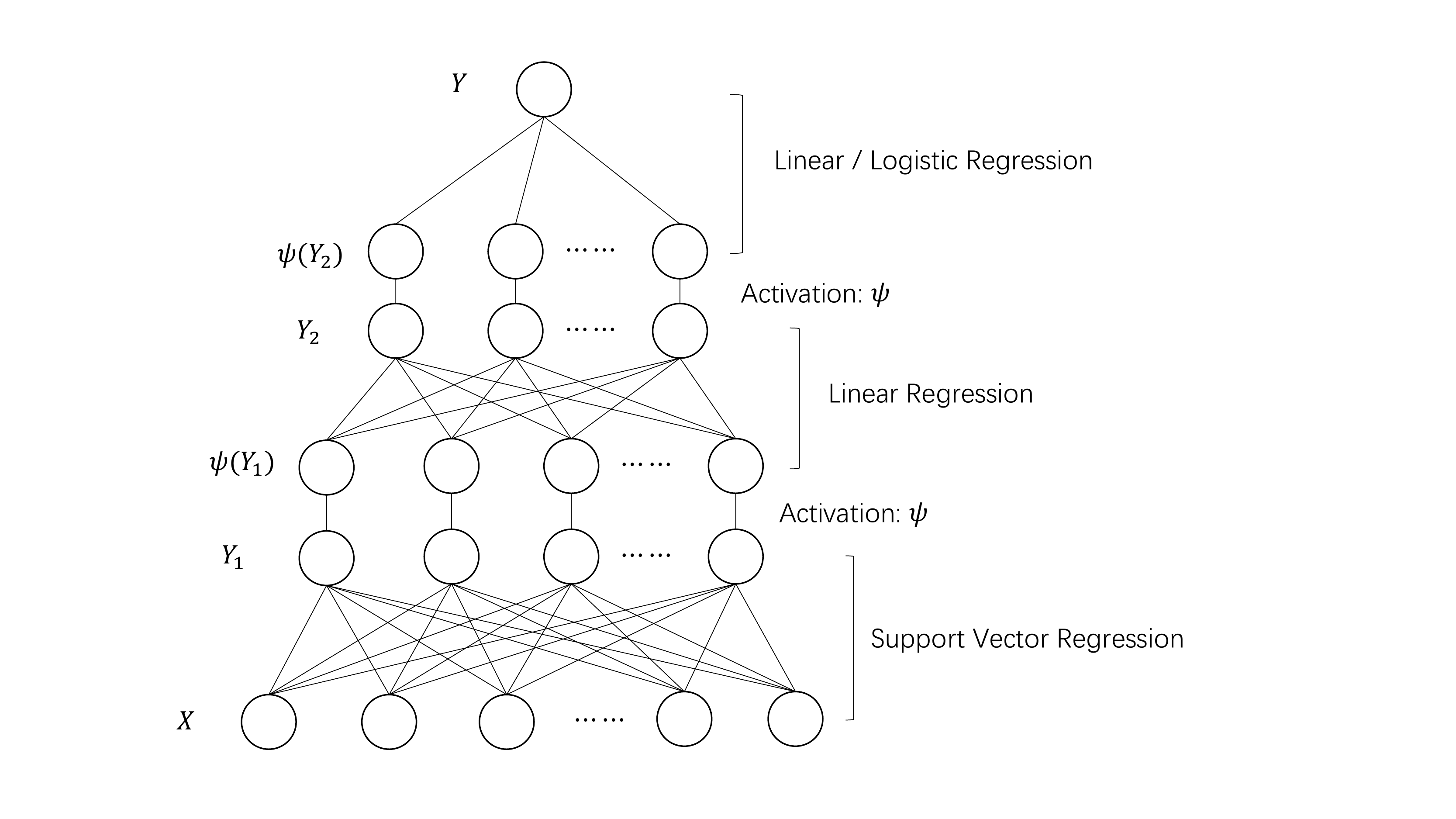}
\caption{An illustrative plot of K-StoNet}
\label{kstonetplot}
\end{figure}
 
To tackle this issue, we introduce
a K-StoNet model (depicted by Figure \ref{kstonetplot}) by adding 
auxiliary noise to $\tilde{\bY}_i$'s, $i=1,2,\ldots, h$, in (\ref{KDNNeq}). The resulting model is given by 
\begin{equation} \label{Kstoneteq}
\begin{split}
{\bY}_1 & =\bb_1+\bbeta \phi(\bX)+\be_1, \\
{\bY}_i & =\bb_i+\bw_i \Psi({\bY}_{i-1})+\be_i, \quad i=2,3,\ldots,h, \\
\bY & =\bb_{h+1}+\bw_{h+1} \Psi({\bY}_h)+\be_{h+1},
\end{split}
\end{equation}
where $Y_1,Y_2,\ldots,Y_h$ are latent variables. To complete the model specification, 
 we assume that $\be_i \sim N(0,\sigma_i^2 I_{m_i})$ for $i=2,3,\ldots,h,h+1$, 
 and each component of $\be_1$ is independent and identically distributed with the 
 density function given by 
 \begin{equation} \label{SVRdist}
 f(x) =\frac{C}{2(1+C \epsilon)}e^{-C |x|_{\varepsilon}},
 \end{equation}
 where $|x|_{\epsilon}=\max(0,|x|-\varepsilon)$ is an $\varepsilon$-intensive loss function, and $C$ is a scale parameter. It is known that 
 this distribution has mean 0 and variance 
 $\frac{2}{C^2}+\frac{\varepsilon^2(\varepsilon C+3)}{3(\varepsilon C+1)}$. 
 For classification networks, the last equation of (\ref{Kstoneteq}) is replaced by a generalized linear model (GLM), for which the parameter $\sigma_{h+1}^2$ plays the role of temperature for the binomial or multinomial distribution formed at the output layer. 
 \textcolor{black}{In summary, $\{C, \varepsilon, \sigma_2^2,\ldots,\sigma_h^2, \sigma_{h+1}^2\}$ work together to control the variation of the latent variables $\{\bY_1,\ldots,\bY_h\}$
 as discussed in Section \ref{hyperpsect}. }   
 As shown later, such specifications for the auxiliary noise enable the K-StoNet parameters to be estimated by solving a series of  convex optimization problems and the prediction uncertainty to be easily assessed via a recursive formula.

To establish that K-StoNet is a valid approximator to KNN, i.e., asymptotically they have the same loss function, some assumptions need to be imposed on the model. To indicate their dependence on the training sample size $n$, we redenote $C$ by $C_n$,  $\varepsilon$ by $\varepsilon_n$, and $\sigma_{i}$ by $\sigma_{n,i}$ for $i=2,3,\ldots,h+1$. For
  (\ref{SVRdist}), we assume $\varepsilon_n \leq 1/C_n$ holds as $n \to \infty$. As in KNN, we let 
$\btheta$ denote the parameter vector of K-StoNet, and let $d_{\theta}$ denote the dimension of $\btheta$. 
Since, for the KNN, any local minimum is also a global minimum, we can restrict $\Theta$ to a compact set which is large enough such that one local minimum is contained. This is essentially a technical condition. In practice, if a local convergence algorithm is used for training the KNN, it is then equivalent to set $\Theta=\mathbb{R}^{d_{\theta}}$, as the regions beyond a neighborhood of the starting point will never be visited by the algorithm.  

\begin{assump} \label{ass2}
(i) $\Theta$ is compact, which can be contained in a $d_{\theta}$-ball centered at the origin and of radius $r$;
(ii) $\mbE (\log \pi(\bY|\bX,\btheta))^2 <\infty$ for any  $\btheta \in \Theta$;
(iii) the activation function $\psi(\cdot)$ is $c'$-Lipschitz continuous for some constant $c'$;
(iv) the network's depth $h$ and widths $m_i$'s are all allowed to increase with $n$;
and (v) $\sigma_{n,h+1}=O(1)$, 
and $m_{h+1} (\prod_{i=k+1}^h m_i^2)m_{k} \sigma_{n,k}^2
\prec \frac{1}{h}$ for $k \in \{1,2,\ldots,h\}$,
where  $\sigma_{n,1}=1/C_n$.
\end{assump}

\textcolor{black}{Assumption A2-(ii) is the regularity condition for the distribution of $\bY$. Assumption A2-(iii) can be satisfied by many activation functions such as {\it tanh}, {\it sigmoid} and {\it softplus}. Assumption A2-(v) constrains the size of the noise added to each hidden layer such that the K-StoNet has asymptotically the same loss function as the KNN when the training sample size becomes large, where the factor $m_{h+1} (\prod_{i=k+1}^h m_i^2)m_{k}$ is derived in the proof of Theorem 2 and its square root can be interpreted as the amplification factor 
of the noise $\be_k$ at the output layer. }

As stated in Assumption \ref{ass4}, the SVR in K-StoNet is required to work with a universal kernel such as RBF. By  \citep{UnivKernel2006,Hammer2003,RBFUniv1991},
such a SVR possesses the universal approximation capability, so does K-StoNet. Therefore,
 K-StoNet is not necessarily very deep or wide, while having any continuous function approximated arbitrarily well as the training sample size $n \to \infty$. 
 For this reason, 
 we may restrict the depth $h=O(1)$, and restrict $m_1=o(\sqrt{n})$ and thus $m_i=o(\sqrt{n})$ for all $i=2,3,\ldots,h$ due to the pyramidal structure of K-StoNet. 
 The universal approximation property of SVR is quite different from that of the neural networks. The   former depends on the training sample size, while the latter depends on the network size. 
  The K-StoNet lies in the between of them.

Theorem \ref{thm:2} shows that the K-StoNet and KNN have asymptotically the same training loss function, whose proof is given in the Appendix.  

\begin{theorem} \label{thm:2} Suppose Assumption \ref{ass2} holds. Then the K-StoNet (\ref{Kstoneteq}) and the KNN (\ref{KDNNeq}) have asymptotically the same loss function, i.e.,  as $n\to \infty$,
\begin{equation} \label{equeq}
\small
 \sup_{\btheta \in \Theta} \big| \frac{1}{n} \sum_{i=1}^n \log\pi(\bY^{(i)},\bY^{(i)}_{\rm mis}  |\bX^{(i)},\btheta) -\frac{1}{n} \sum_{i=1}^n \log\pi(\bY^{(i)} |\bX^{(i)},\btheta) \big| \stackrel{p}{\to} 0, 
\end{equation}
where $\bY_{\rm mis}^{(i)}=(\bY_1,\bY_2,\ldots,\bY_h)$ denotes the collection of latent variables in (\ref{Kstoneteq}), and $\stackrel{p}{\to}$ denotes convergence in probability. 
\end{theorem}
 
Let $Q^*(\btheta)=\mbE (\log \pi(\bY|\bX,\btheta))$, where the expectation is taken with respect to the joint distribution $\pi(\bX,\bY)$.
 By Assumption \ref{ass2}-(i) \& (ii), and the law of large numbers, 
 \begin{equation} \label{equeq2}
\frac{1}{n} \sum_{i=1}^n \log \pi(\bY^{(i)}|\bX^{(i)},\btheta) -Q^*(\btheta) \stackrel{p}{\to} 0,
\end{equation}
holds uniformly over $\Theta$.
Further, we make the following assumptions for 
$Q^*(\btheta)$:

\begin{assump} \label{ass3} (i)
 $Q^*(\btheta)$ is continuous in $\btheta$ and 
  uniquely maximized at $\btheta^*$; (ii)  
  for any $\epsilon>0$,  $\sup_{\btheta \in \Theta \setminus B(\epsilon)} 
  Q^*(\btheta)$ exists, where 
   $B(\epsilon)=\{\btheta: \|\btheta-\btheta^*\| < \epsilon\}$, and  $\delta=Q^*(\btheta^*)-
    \sup_{\btheta \in \Theta\setminus B(\epsilon)} Q^*(\btheta)>0$.
\end{assump}
\textcolor{black}{
Assumption A3 restricts the shape of $Q^*(\btheta)$ around the global maximizer, which cannot be discontinuous or too flat. Given nonidentifiability of the neural network model (see e.g. \cite{SunSLiang2021}), we here have implicitly assumed that each $\btheta$ in the KNN and K-StoNet is unique up to loss-invariant transformations, such as reordering some hidden units and simultaneously changing the signs of some weights and biases. } 

 \begin{lemma}\label{lem2}  Suppose 
 Assumptions \ref{ass2}-\ref{ass3} hold, and 
 $\pi(\bY,\bY_{\rm mis} |\bX,\btheta)$ is continuous 
 in $\btheta$. Let $\hat{\btheta}_n=\arg\max_{\btheta\in \Theta} 
 \big\{ \frac{1}{n} \sum_{i=1}^n \log\pi(\bY^{(i)},\bY^{(i)}_{\rm mis} \big |\bX^{(i)},\btheta) \big \}$.
 Then  $\|\hat{\btheta}_n- \btheta^*\| \stackrel{p}{\to} 0$ as $n\to \infty$.
\end{lemma}

The proof of Lemma \ref{lem2} is given in the Appendix. It implies that the KNN can be trained
by training K-StoNet as the sample size $n$ becomes large.

 \subsection{The Imputation-Regularized Optimization Algorithm} \label{IROsect}
 
To train the K-StoNet, we propose to use 
the imputation-regularized optimization (IRO) algorithm \cite{Liang2018missing}. Consider a 
missing data problem, where $\bXobs$ denotes observed data,  $\bXmis$ denotes  missing data, and  $\bvartheta$ denotes the parameter.
The IRO algorithm aims to find a consistent estimate of $\bvartheta$ by maximizing 
$\mathbb{E} \log\pi(\bXobs,\bXmis|\bvartheta)$, where the expectation is taken with respect to the joint distribution of $(\bXobs,\bXmis)$. Conceptually, this is a little different from the expectation-maximization (EM) algorithm \citep{Dempster1977} and stochastic EM algorithm \citep{CeleuxDiebolt1995}, which aim to estimate $\bvartheta$ by maximizing the marginal likelihood function $\pi(\bXobs|\bvartheta)$. 
 Practically, the IRO algorithm works in similar way to stochastic EM by iterating between an imputation step and an optimization step, but for which a regularization term can be included in the loss function at each optimization step for ensuring the convergence of the estimate under the high-dimensional scenario.

For K-StoNet, the IRO algorithm is to estimate 
$\btheta$ by maximizing $\mathbb{E} \log \pi(\bY,\bYmis|\bX,\btheta)$, which is equivalent to maximizing 
$Q^*(\btheta)=\mathbb{E} \log\pi(\bY|\bX,\btheta)$ as implied by 
(\ref{equeq}) and (\ref{equeq2}). This coincides with the goal of KNN training if the stochastic gradient descent (SGD) algorithm 
is used. 
Let $\hat{\btheta}_n^{(t)}$ denote the estimate of $\btheta$ obtained by the IRO algorithm at iteration $t$. The IRO algorithm 
 starts with an initial guess $\hat{\btheta}_n^{(0)}$ and then iterates between
 the following two steps:
 
\begin{itemize}
\item {\bf I-step}: {\it For each sample $(\bX^{(i)},\bY^{(i)})$, draw $\bYmis^{(i)}$ from the predictive distribution 
  $g(\bYmis|\bY^{(i)}, \bX^{(i)}, \hat{\btheta}^{(t)})$.}  

\item {\bf RO-step}: {\it Based on the pseudo-complete data, find 
   an updated estimate $\hat{\btheta}_n^{(t+1)}$ by 
   minimizing the penalized loss function, i.e.,
   \begin{equation} \label{IRoeq1}
   \small
    \hat{\btheta}_n^{(t+1)} =\arg\min\left\{  -\frac{1}{n} \sum_{i=1}^n \log\pi(\bY^{(i)},\bY^{(i)}_{\rm mis} \big |\bX^{(i)},\btheta) 
    + P_{{\lambda}_n}(\btheta)\right\},
\end{equation}
   where the penalty function $P_{{\lambda}_n}(\btheta)$ is chosen such that 
   $\hat{\btheta}_n^{(t+1)}$ forms a consistent estimate  of {\small 
   $\btheta_*^{(t)}=\arg\max_{\btheta}  \mbE_{\btheta_n^{(t)}} \log \pi(\bY,\bYmis|\bX,\btheta)
   = \int \log \pi(\bYmis, \bY|\bX,\btheta)) g(\bYmis|\bY,\bX,\btheta_n^{(t)}) 
   \pi(\bY|\btheta^*,\bX) \pi(\bX)$ $d\bYmis d\bY d\bX$}, 
   $\btheta^*$ denotes the true parameter of the model, and $\pi(\bX)$  denotes the density function of $\bX$. }
 \end{itemize}

For the K-StoNet, the joint distribution 
$\pi(\bY,\bYmis|\bX,\btheta)$ can be factored as
\begin{equation} \label{jointeq}
\small
\pi(\bY,\bYmis|\bX,\btheta) = \pi(\bY_1|\bX,\tilde{\bw}_1) [\prod_{i=2}^{h} \pi(\bY_i|\bY_{i-1}, \tilde{\bw}_i)] \pi(\bY|\bY_{h},\tilde{\bw}_{h+1}). 
\end{equation}
 Therefore, the optimization in (\ref{IRoeq1}) can be executed separately for each of the hidden and output layers with an appropriately specified penalty function. That is,  {\it K-StoNet can be trained by solving a series of lower dimensional optimization problems}.

 For the first hidden layer, the RO-step is reduced to solving a SVR for each hidden unit. 
 As described in \cite{smola2004tutorial}, 
 the parameter $\bbeta$ in (\ref{Kstoneteq}) can be
 estimated by solving a regularized optimization problem:
 \begin{equation} \label{SVRregularization}
 \begin{split}
  \arg\min_{\bbeta, \bb_1} \frac{1}{2}||\bbeta||_2^2 + \frac{\tilde{C}_n}{n} \sum_{i=1}^n |\bY_1^{(i)} - \bbeta \phi(\bX^{(i)}) - \bb_1 |_{\varepsilon},
\end{split}
 \end{equation}
 where the first term represents a penalty function, and $\tilde{C}_n$ represents the regularization parameter. We can set $\tilde{C}_n=C_n$ given in (\ref{SVRdist}), but not necessarily. In general, their values should make Assumptions \ref{ass2}-(v) and \ref{ass4}-(i) hold. 
  The consistency of the SVR estimator $\hat{\bbeta}^T\phi(\bx)+\hat{\bb}_1$, which is the basic requirement by the IRO algorithm,
  has been established in \cite{SVRconsistency2007} by assuming that a universal kernel \cite{Steinwartkernel2002,UnivKernel2006} such as RBF is used in (\ref{SVRregularization}).
  Equivalently, this is to reparameterize the SVR layer by kernel-based regression. By the kernel representer theorem \citep{Wahba1990, Scholkopf2001}, the solution to the regularized optimization problem (\ref{SVRregularization}) leads to the representer of the first equation of (\ref{Kstoneteq}) as  
 \begin{equation} \label{May28eq}
\bY_1=\hat{\bb}_{1}+\sum_{i=1}^n 
 \hat{\bw}_{1}^{(i)} K(\bX^{(i)},\bX)+\be_1.
\end{equation}
 In what follows, we will use  

 \textcolor{black}{$\check{\bw}_1=(\hat{\bw}_1^{(1)}, \ldots, \hat{\bw}_1^{(n)},\hat{\bb}_1)$ }
 to denote the estimator for the parameters of the SVR layer. 
 
 For other hidden layers, the RO-step is reduced to \textcolor{black}{solving}
 a linear regression for each hidden unit using a regularization method. To ensure convexity of the resulting objective function, a Lasso penalty \cite{Tibshirani1996} can be used. Alternatively, some nonconvex amenable penalties with vanishing derivatives away from the origin, such as the SCAD \citep{FanL2001} and MCP \citep{Zhang2010}, can also be used. As shown in  \cite{loh2017support}, for such nonconvex amenable penalties, any stationary point in a compact region around the true regression coefficients can be used to consistently estimate the parameters and recover the support of the underlying true regression.

For the output layer, the RO-step is reduced to \textcolor{black}{solving} a
multinomial logistic  or multivariate linear regression, depending on the problem under consideration. The Lasso, SCAD and MCP penalties can again be used for them by the theory of \cite{loh2017support}. In practice, this step can also be simplified to solving a linear or logistic regression for each output unit by ignoring the correlation between different components of $\bY$.

In summary, we have the pseudo-code given in Algorithm \ref{IRO} for training K-StoNet, where
\textcolor{black}{${\check{\bw}}_i^{(t)}$} 
denotes the estimate of the parameters for the layer $i$ at iteration $t$,  
$(\bY_0^{(s)},\bY_{h+1}^{(s)})=(\bX^{(s)},\bY^{(s)})$ denotes a training sample,
  $(\bY_1^{(s,t)},\ldots,\bY_h^{(s,t)})$ denotes the latent variables imputed for
   training sample $s$ at iteration $t$.
 For convenience, we occasionally use the notation $\bY_0^{(s,t)}=\bY_0^{(s)}$ and
 $\bY_{h+1}^{(s,t)}=\bY_{h+1}^{(s)}$. 
 
\begin{algorithm}[htbp]
\caption{The IRO Algorithm for K-StoNet Training}
\label{IRO}
\SetAlgoLined
 {\bf Input}: the total iteration number $T$, the Monte Carlo step number $t_{HMC}$, and the learning rate sequences $\{\epsilon_{t,i}: t=1,2,\ldots,T; i=1,2,\ldots,h+1\}$.
 
 {\bf Initialization}: Randomly initialize the network parameters 
\textcolor{black}{ $\hat{\btheta}_n^{(0)}=({\check{\bw}}_1^{(0)},\ldots,{\check{\bw}}_{h+1}^{(0)})$.}

 \For{t=1,2,\dots,T}{
   {\bf STEP 1. Backward Imputation:} 
   For each observation $s$, impute the latent variables in the order from layer $h$ to layer 1. More explicitly, impute $\bY_i^{(s,t)}$ from the  distribution 
   $\pi(\bY_i^{(s,t)}|
   \bY_{i+1}^{(s,t)},\bY_{i-1}^{(s,t)},
   {\check{\bw}}_i^{(t-1)}, {\check{\bw}}_{i+1}^{(t-1)}) \propto$
   $\pi(\bY_i^{(s,t)}|\bY_{i-1}^{(s,t)},{\check{\bw}}_i^{(t-1)}) \pi(\bY_{i+1}^{(s,t)}|\bY_i^{(s,t)},{\check{\bw}}_{i+1}^{(t-1)})$  
   by running HMC in $t_{HMC}$ steps, where
   $\pi(\bY_1^{(s,t)}|\bX^{(s)},{\check{\bw}}_1^{(t-1)})$ 
   can be expressed based on 
   (\ref{May28eq}).
 
  {\bf (1.1) Initialization}: Initialize $\bv_i^{(s,0)} =\boldsymbol{0}$, and  initialize $\bY_i^{(s,t,0)}$ by KNN, i.e., calculating $\bY_i^{(s,t,0)}$ for $i=1,2,\ldots,h$ in (\ref{Kstoneteq}) by setting the random errors to zero.
  
  {\bf (1.2) Imputation}:
  \For{k = 1, 2, \dots, $t_{HMC}$}{
   \For{i = h,h-1,\dots, 1}{
 \textcolor{black}{ 
 \begin{equation}  \label{imputationeq} 
  \small
     \begin{split}
    \bv_i^{(s,k)} & = (1 - \alpha)\bv_i^{(s,k-1)}+ \epsilon_{t,i} \nabla_{\bY_i^{(s,t, k-1)}} \log  \pi(\bY_i^{(s,t, k-1)}| 
   \bY_{i-1}^{(s,t, k-1)}, {\check{\bw}}_i^{(t-1)}) \\
       &  + \epsilon_{t,i} \nabla_{\bY_i^{(s,t, k-1)}}
         \log \pi(\bY_{i+1}^{(s,t, k)}| 
         \bY_i^{(s,t,k-1)}, {\check{\bw}}_{i+1}^{(t-1)}) + \sqrt{2\alpha\epsilon_{t,i}} \bz^{(s,t,k)},\\
    \bY_i^{(s,t,k)} & = \bY_i^{(s,t,k-1)} + \bv_i^{(s,k)}, 
    \end{split}
    \end{equation}
    }
    where $\bz^{(s,t,k)}\sim N(0, \bI_{m_i})$, $\epsilon_{t,i}$ is the learning rate,  
    and $1-\alpha$ is the momentum decay factor  ($\alpha = 1$ corresponds to Langevin Monte Carlo).
  } }
 {\bf (1.3) Output}: Set $\bY_i^{(s,t)} = \bY_i^{(s,t,t_{HMC})}$ for $i=1,2,\dots, h$.
  
   {\bf STEP 2. Parameter Updating:} 
   Update the estimates
      $({\check{\bw}}_1^{(t-1)},{\check{\bw}}_2^{(t-1)},\ldots,{\check{\bw}}_{h+1}^{(t-1)})$ 
   by solving $h+1$ penalized multivariate
   regressions separately. \\
{\bf (2.1) SVR layer}: 
\begin{equation} \label{optsolver1}
{\check{\bw}}_1^{(t)}= \arg\min_{\bbeta,\bb_1} \left\{ \frac{\tilde{C}_{n,1}^{(t)}}{n} \sum_{s=1}^{n} \| |\bY_{1}^{(s,t)}  -\bbeta^T \phi(Y_{0}^{(s,t)})  -\bb_1|_{\varepsilon} \|_1+  \frac{1}{2}\|\bbeta \|_2^2 \right\},
\end{equation}

 where $\tilde{C}_{n,1}^{(t)}$ is the regularization parameter used at iteration $t$.
 
{\bf (2.2) Regression layers}: \For{i=2,3,\ldots,h+1}{
\begin{equation} \label{optsolver2}
{\check{\bw}}_i^{(t)}=\arg \min_{\bw_i,\bb_i} \left\{ \frac{1}{n} \sum_{s=1}^{n} \|\bY_{i}^{(s,t)}-\bw_i \psi_i(\bY_{i-1}^{(s,t)})  -\bb_i\|_2^2+ P_{\lambda_{n,i}^{(t)}}(\tilde{\bw}_i) \right\},
\end{equation}
 where $\lambda_{n,i}^{(t)}$ is the regularization parameter used for layer $i$ at iteration $t$.  
 }
{\bf (2.3) Output:} Denote the updated estimate by
$\hat{\btheta}^{(t)}=({\check{\bw}}_1^{(t)},
\ldots, {\check{\bw}}_{h+1}^{(t)})$.
  }
\end{algorithm}

For Algorithm \ref{IRO}, we have a few remarks: 

\begin{itemize}
    \item  The Hamiltonian Monte Carlo (HMC) algorithm \cite{HMC1987,HMC-Neal2011,Cheng2018underdamped} is employed in the backward imputation step. Other MCMC algorithms such as Langevin Monte Carlo \citep{Rossky1978BrownianDA} and  the Gibbs sampler \cite{GemanG1984} can also be employed there.
     
    \item In the parameter update step, a Lasso penalty \cite{Tibshirani1996} is used in (\ref{optsolver2}) to induce the sparsity of StoNet, while ensuring convexity of the minimization problems. Therefore,
    {\it K-StoNet is trained by solving a series of convex optimization problems.} Note that the minimization in  (\ref{optsolver1}) is known as a convex quadratic programming problem \cite{Vapnik2000, BalaSVR2013}. Although solving the convex optimization problems is more expensive than a single gradient update, the IRO algorithm converges very fast, usually within tens of iterations. 
   
   \item  The major computational cost of K-StoNet comes from the SVR step when the sample size is large.  The computational complexity for solving an SVR is $O(n^2p+n^3)$, and that for solving a linear/logistic regression is bounded by $O(n m_1^2+m_1^3)$, while $m_1 \prec n^{1/2}$ is usually recommended. A scalable SVR solver will accelerate the computation of K-StoNet substantially. This issue will be further discussed at the end of the paper.
   
    \item If $m_1 \prec \sqrt{n}$ holds, then the penalty in
   (\ref{optsolver2}) can be simply removed for computational simplicity, while ensuring asymptotic normality of the resulting regression coefficient estimates by \citep{Portnoy1988}.  
\end{itemize}
     
 Like the stochastic EM algorithm, the IRO algorithm generates two interleaved Markov chains:
 \[
  \hat{\btheta}_n^{(0)} \to (\bY_1^{(1)},\ldots,\bY_h^{(1)}) \to \hat{\btheta}_n^{(1)} \to 
 (\bY_1^{(2)},\ldots,\bY_h^{(2)}) \to \hat{\btheta}_n^{(2)} \to 
 \cdots,
 \]
 whose convergence theory has been studied in  \cite{Liang2018missing}.
  To ensure the convergence  of the Markov chains in K-StoNet training, we make following assumptions for the regularization parameters used in (\ref{optsolver1}) and (\ref{optsolver2}):
  
  \begin{assump} \label{ass4}  
  (i) A universal kernel such as RBF is used in the SVR layer, and for each $t\in \{1,2,\ldots,T\}$, $1 \prec \tilde{C}_{n,1}^{(t)} \prec \sqrt{n}$ holds;
  and (ii)  
  for each $t\in \{1,2,\ldots,T\}$ and each $i\in \{2,3,\ldots,h+1\}$, 
  $\sup_{\tilde{\bw}_i \in \Theta_i} P_{\lambda_{n,i}}^{(t)}(\tilde{\bw}_i) \to 0$ holds as $n\to \infty$, where $\Theta_i$ denotes the sample space of $\tilde{\bw}_i$.
  \end{assump}
  
  Assumption \ref{ass4}-(i) ensures consistency of the regression function estimator in the SVR step by Theorem 12 of \cite{SVRconsistency2007}.
  Assumptions \ref{ass4}-(ii) ensures consistency of the weight estimators in the output and other hidden layers. For the Lasso penalty, we can set $P_{\lambda_{n,i}^{(t)}}(\tilde{\bw}_i)=\lambda_{n,i}^{(t)} \|\tilde{\bw}_i\|_1$ and $\lambda_{n,i}^{(t)}=O(\sqrt{\log (m_{i-1})/n})$ for any $t\in\{1,2,\ldots,T\}$ and $i\in\{2,3,\ldots,h+1\}$. Since $\Theta$ is bounded as assumed in \ref{ass2}-(i), \ref{ass4}-(ii) is satisfied. 
    In summary, we have the following theorem which is essentially a restatement of  Theorem 4 and Corollary 3 of \cite{Liang2018missing} and therefore whose proof is omitted.
    
     \begin{theorem} \label{thm:3} (Consistency)
    Suppose that Assumptions \ref{ass1}-\ref{ass4} hold and, further, the general regularity conditions on missing data (given in \cite{Liang2018missing}) hold. 
    Then for sufficiently large $n$, sufficiently large $T$, and almost every $(\bX,\bY)$-sequence, $\|\hat{\btheta}_n^{(T)} -\btheta^*\| \stackrel{p}{\to} 0$ and
    $\|\frac{1}{T}\sum_{t=1}^T \hat{\btheta}_n^{(t)}-\btheta^*\| \stackrel{p}{\to} 0$. In addition, for any 
    Lipschitz continuous function  $\zeta(\cdot)$ on $\Theta$, $\|\frac{1}{T}\sum_{t=1}^T\zeta(\hat{\btheta}_n^{(t)})-\zeta(\btheta^*)\| \stackrel{p}{\to} 0$.
    \end{theorem}

\textcolor{black}{As implied by Theorems \ref{thm:1}-\ref{thm:3} and Lemma \ref{lem2}, 
$\hat{\btheta}_n^{(t)}$ asymptotically converges to a  global optimum of the KNN.} 
For each $\hat{\btheta}_n^{(t)}$,  when making predictions, one can simply calculate the output in (\ref{Kstoneteq}) by ignoring the auxiliary noise, i.e., treating $\hat{\btheta}_n^{(t)}$ as the weights of a KNN. 
In this way, K-StoNet can be viewed as a tool for training the KNN, although it means more than that.

\subsection{Hyperparameter Setting} \label{hyperpsect}

As mentioned previously, DNN is often over parameterized to avoid getting trapped into a poor local minimum. In contrast, as implied by Theorems \ref{thm:1}-\ref{thm:3}, the local minimum trap is not an issue to K-StoNet any more. This, together with the universal approximation property of K-StoNet and the parsimony principle of statistical modeling, suggests that a small K-StoNet might work well for complex problems.  
As shown in Section \ref{numericalII}, the K-StoNet with a single hidden layer and a small number of hidden units works well for many complex datasets. 

Other than the network structure, the performance of K-StoNet also depends on the network hyperparameters as well as the hyperparameters introduced by the IRO algorithm. The former include $C_n$, $\varepsilon_n$ and $\sigma_{n,k}$'s for $k=2,\ldots,h+1$. The latter include  the learning rates and iteration number used in HMC backward imputation and the regularization parameters used in solving the optimizations (\ref{optsolver1}) and (\ref{optsolver2}). 
The hyperparameters $C_n$, $\varepsilon_n$ and $\sigma_{n,k}$'s control the variation of the latent variables and thus the 
variation of $\btheta_n^{(T)}$ by the theory developed in \cite{Liang2018missing} and \cite{Nielsen2000}. 
\textcolor{black}{In general, setting the latent variables to have slightly large variations can facilitate the convergence of the training process. On the other hand, as required by Assumption \ref{ass2}-(v), we need to control the variations of the latent variables sufficiently small for ensuring the convergence of K-StoNet to a global minimum of the corresponding KNN by noting the stochastic optimization nature of the IRO algorithm. } Assumption \ref{ass2}-(v) provides a clue for setting the network hyperparameters.  
\textcolor{black}{Here we would like to note that when $1/C_n$ and $\sigma_{n,i}^2$'s are set to be very small, to ensure the stability of the algorithm, we typically need to adjust the learning rate $\epsilon_{t,i}$'s to be very small as well such that their effects on the drift term of (\ref{imputationeq}) can be canceled or partially canceled. Meanwhile, to compensate the negative effect of the reduced learning rate on the mobility of the Markov chain, we need to lengthen the MCMC iterations, i.e., increasing the value of $t_{HMC}$,  appropriately.}
Finally, we note that setting $\sigma_{n,i}$'s in the monotonic pattern $\sigma_{n,h+1} \geq \sigma_{n,h} \geq \cdots \geq \sigma_{n,2} \geq 1/C_n$ is generally unnecessary, as long as their values have been in a reasonable range.  

In our experience, the performance of K-StoNet is not very sensitive to these hyperparameters as long as they are set in an appropriate range.  As shown in Appendix \ref{settingsection}, which collects all parameter settings of K-StoNet used in this paper, many examples share the same parameter setting.

\section{Illustrative Examples} \label{numericalI}

This section contains two examples. The first example demonstrates that K-StoNet indeed avoids local traps in training, and the second example demonstrates the performance of K-StoNet in the large-$n$-small-$p$ scenario that DNN typically works in. 
\textcolor{black}{K-StoNet is compared with DNN and KNN. For the KNN, the kernel representer given by  equation (\ref{KNN_kernel_form}) is used as the first hidden layer, and the kernel is set to be the same as that used by K-StoNet.}\footnote{Code for running experiments can be found in https://github.com/sylydya/A-Kernel-Expanded-Stochastic-Neural-Network}

\subsection{A full row rank example} \label{SimulationI} 

The dataset was generated from a two-hidden layer neural network with structure 1000-5-5-1.  
The input variables $x_{1},\ldots,x_{1000}$ were 
generated by independently simulating the variables $e, z_{1},\dots, z_{1000}$ from the standard Gaussian distribution and then setting $x_{i}=\frac{e+z_{i}}{\sqrt{2}}$. In this way, all the input variables are mutually correlated with a correlation coefficient of 0.5. The response variable  was generated by setting
\begin{equation} \label{roweq}
\by = \bw_3 \tanh (\bw_2 \tanh(\bw_1 \bx ) )  + \bepsilon,
\end{equation}
where $\bw_1 \in \mR^{5\times1000}$,  $\bw_2 \in \mR^{5\times 5}$  and $\bw_3 \in \mR^{1\times 5}$ represent the weights at 
different layers of the neural network, $\tanh(\cdot)$ is the hyperbolic tangent function, and the random error $\bepsilon \sim N(0,1)$.  Each elements of $\bw_i$'s was randomly sampled from the set  $\{-2,-1,1,2\}$. 
The full dataset consisted of $1000$ training samples and 1000 test samples.

We first refit the model (\ref{roweq}) using SGD. Since the training samples form a full row rank matrix of size $n=1000$ by $p=1001$ (including the bias term), SGD will not get trapped into a local minimum by Lemma \ref{lem1}.  SGD was run for 2000 epochs with a mini-batch size of 100 and a constant learning rate of 0.005. 
\textcolor{black}{Figure \ref{Simulation_high_dim_update} (upper panel) indicates that SGD indeed converges to a global optimum.}
 
For K-StoNet, we tried a model with one hidden layer and 5 hidden units. The model was trained by IRO  for 40 epochs.  
Since all training samples were used at each iteration, an iteration is equivalent to an epoch for K-StoNet. \textcolor{black}{
We also tried a KNN model for this example, which has the same structure as K-StoNet. 
The KNN was trained using SGD with a constant learning rate of 0.005 for 2000 epochs.}
Figure \ref{Simulation_high_dim_update} (upper panel) \textcolor{black}{compares the training and testing MSE paths of the three models. It
shows that K-StoNet converges to the global optimum in a few epochs; 
while DNN needs over 100 epochs, and KNN needs even more.} 
More importantly, K-StoNet is less bothered by over-fitting, whose prediction performance is stable after convergence has been reached. However, the DNN tends to be over fitted, whose prediction becomes worse and worse as training goes on. \textcolor{black}{The KNN is more stable than DNN in prediction, but worse than K-StoNet. As discussed in Section 2.2, the KNN with the kernel representer for the first hidden layer is not equivalent to K-StoNet in general. This experiment further demonstrates the importance of the stochastic structure introduced in  K-StoNet.}

\begin{figure}[htbp]

\centering
\begin{tabular}{c}
\includegraphics[height=2.75in,width=1.0\linewidth]{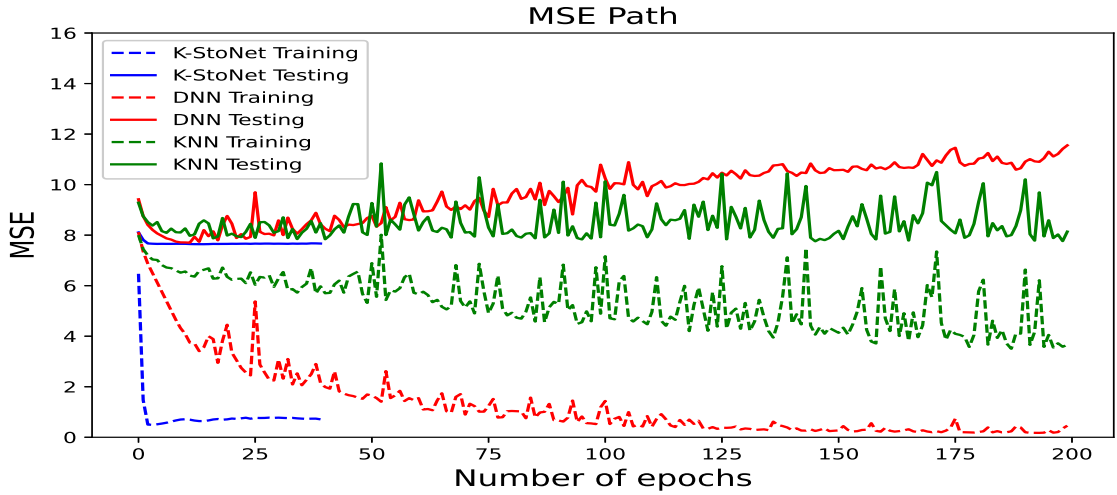} \\
\includegraphics[height=2.75in,width=1.0\linewidth]{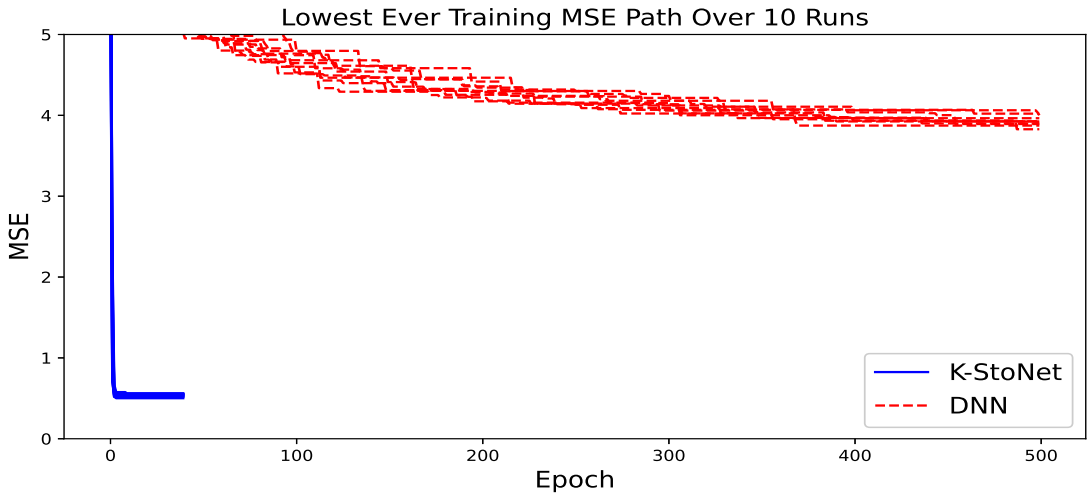}
\end{tabular}
\caption{{\it Upper Panel}: paths of the mean squared error (MSE) produced by K-StoNet and an unregularized DNN for one simulated dataset; and {\it Lower Panel}: best MSE (by the current epoch) produced by SGD for a regularized DNN and K-StoNet over 10 runs. }
\label{Simulation_high_dim_update}
\end{figure}

To explore the loss surface of the regularized DNN, we re-trained the true DNN model (\ref{roweq}) using SGD with a Lasso penalty ($\lambda=0.1$) imposed on all the weights. 
SGD was run for 500 epochs with a mini-batch size of 100 and a constant learning rate of 0.005. The run was repeated for 10 times. For comparison, K-StoNet was also retained for 10 times. Their convergence paths were shown in the lower panel of Figure \ref{Simulation_high_dim_update}. The comparison shows that the regularized DNN might suffer from local traps (different runs converged to different MSE values), 
while K-StoNet does not although its RO step also involves penalty terms.  \textcolor{black}{ According to the theory developed in \cite{Liang2018missing},  
the convergence of the IRO algorithm requires a consistent estimate of $\btheta_*^{(t)}$ to be obtained at each RO step and an appropriate penalty term is allowed for obtaining the consistent estimate. For K-StoNet, to ensure a consistent 
estimate to be obtained at each parameter updating step, we impose a $L_2$-penalty on the SVR layer and a Lasso penalty on the regression layers. For both of them, the resulting loss functions are convex, and the corresponding 
consistent (optimal) estimates are uniquely determined. 
Then, by Theorems \ref{thm:1}--\ref{thm:3} and Lemma \ref{lem2}, the convergence of K-StoNet to the global optimum is asymptotically guaranteed.}

In the previous example, the data was generated from a DNN model. Even working with the true structure, DNN is still inferior to K-StoNet in training and prediction. To further demonstrate the advantage of K-StoNet, we generated a dataset from a KNN model. The dataset consisted of 5000 training samples, where the input variables $\bx \in \mR^{5}$ with each component being a standard Gaussian random variable and a mutual correlation coefficient of 0.5.
Let $\bk = (K(\bx^{(1)}, \bx), \dots, K(\bx^{(5000)}, \bx) )^{T} \in \mR^{5000}$, where $K(\cdot, \cdot)$ is the RBF kernel and $\bx^{(i)}$ denotes the $i$th training sample. The response variable was generated by
\[
\by = \bw_2 \tanh(\bw_1 \bk + \bb_1) + \bb_2 + \bepsilon,
\]
where $\bw_2\in \mR^{1\times 5}$, $\bw_1 \in \mR^{5 \times 5000}$, $\bb_1 \in \mR^{5}$, $\bb_2 \in \mR$, and 
$\bepsilon \sim N(0,1)$. The components of $\bw_2$ and $\bb_2$ were randomly generated from $N(0,1)$, and $\bw_1$ and $\bb_1$ were the dual parameters of five SVR models with the above training samples as input and some vectors randomly generated from $N(0, I_5)$ as response. We set $C = 5$ and $\epsilon = 0.01$ for the SVR model. We also generated another 5000 samples from the same model as test data. Then we modeled the data by K-StoNet with one hidden layer, for which we tried the cases with 5 hidden units and 10 hidden units and set  $C=5$ and $\epsilon=0.01$ for each SVR. For comparison, we tried \textcolor{black}{a KNN with 5 hidden units,} a DNN with one hidden layer and 50 hidden units, and a DNN with 3 hidden layers and 50 hidden units on each layer. Figure \ref{knn_model} shows the training and testing paths of the five
models. For this example, K-StoNet achieved a training MSE about 1.0 and significantly outperformed the \textcolor{black}{DNN and KNN models} in prediction.

\begin{figure}[htbp]
\centering
\includegraphics[height=3.0in,width=1.0\textwidth]{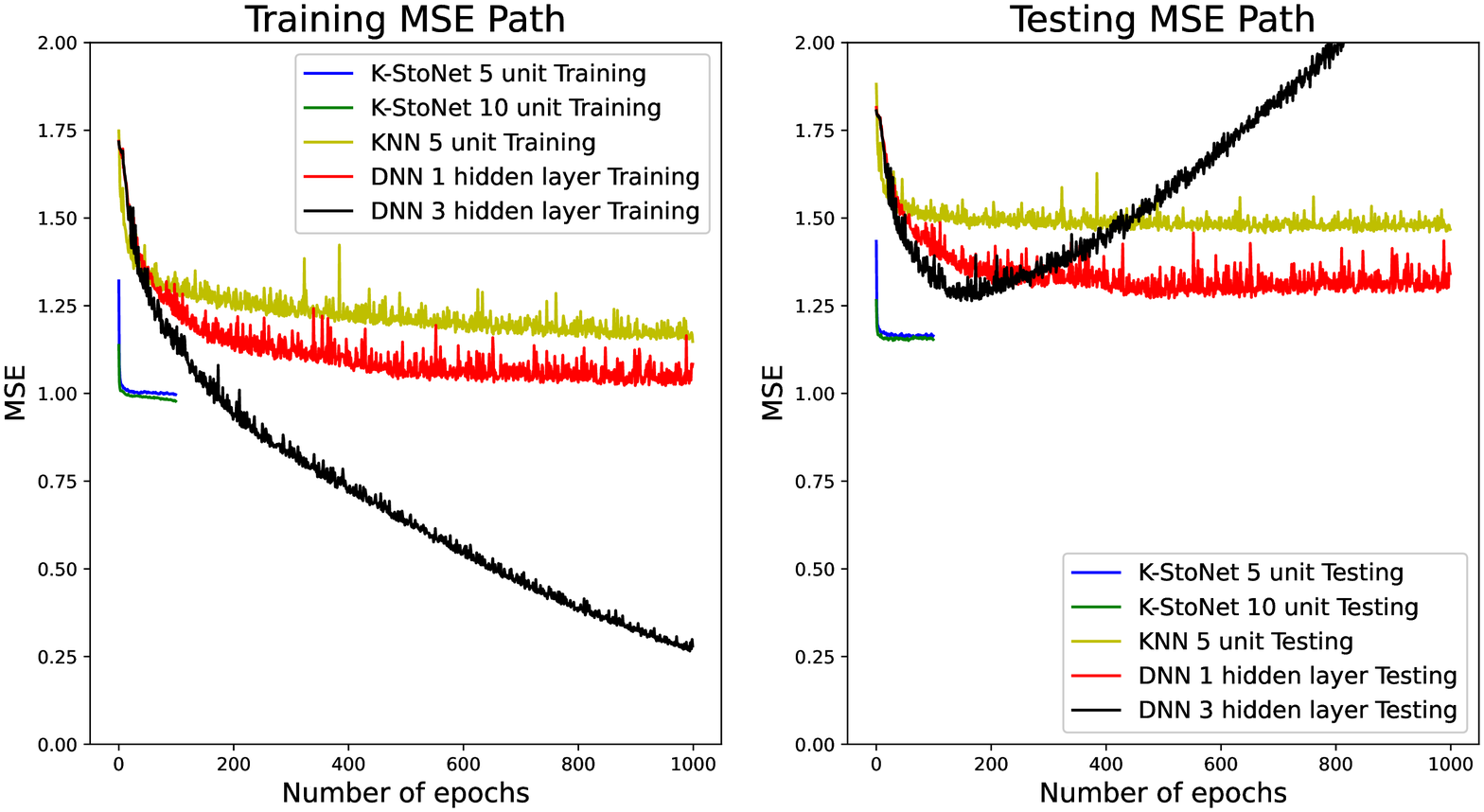}
\caption{MSE paths produced by two K-StoNets, one KNN, and two DNNs for the data generated from a KNN model: the left plot is for training and the right plot is for testing.}
\label{knn_model}
\end{figure}

 \subsection{A measurement error example}
 \label{SimulationII}
  
This example mimics the typical scenario under which DNN works. We generated 500 training samples and 500 test samples from a nonlinear regression: for each sample $(Y, \bX)$, where $Y\in \mR$ and $\bX=(X_1,\ldots,X_5) \in \mR^5$. The explanatory variables $X_1, \dots, X_5$ were generated such that each follows the standard Gaussian distribution, 
while they are mutually correlated with a correlation coefficient of 0.5.
The response variable was generated from the nonlinear regression
\[
\scriptsize
Y = \frac{5 X_2}{1+X_1^2} + 5\sin(X_3 X_4) + 2 X_5 + \epsilon,
\]
where $\epsilon \sim N(0,1)$. Then each explanatory variable was perturbed by adding a random measurement error independently drawn from $N(0,0.5)$. 

We modeled the data using two different K-StoNets, one with 1-hidden layer and 5 hidden units, and the other with 3-hidden layers and 20 hidden units on each hidden layer. Both models were trained by IRO for 1000 epochs.
For comparison, \textcolor{black}{we also modeled the data by KNNs and DNNs with the same structures as the K-StoNets.}
The KNNs and DNNs were trained by  SGD with momentum for 1000 epochs with a minibatch size of 100, a constant learning rate of 0.005, and a momentum decay factor of 0.9. 
As shown in Figure \ref{Measurement_Error}, the 1-hidden layer DNN and KNN perform stably in both training and testing, while the 3-hidden layer DNN and KNN are obviously over-fitted.
Compared to the DNN and KNN, K-StoNet is resistant to over-fitting, even when an overly large model is employed.

 \begin{figure}[htbp]
\begin{tabular}{c}
\includegraphics[height=3.5in,width=1.0\textwidth]{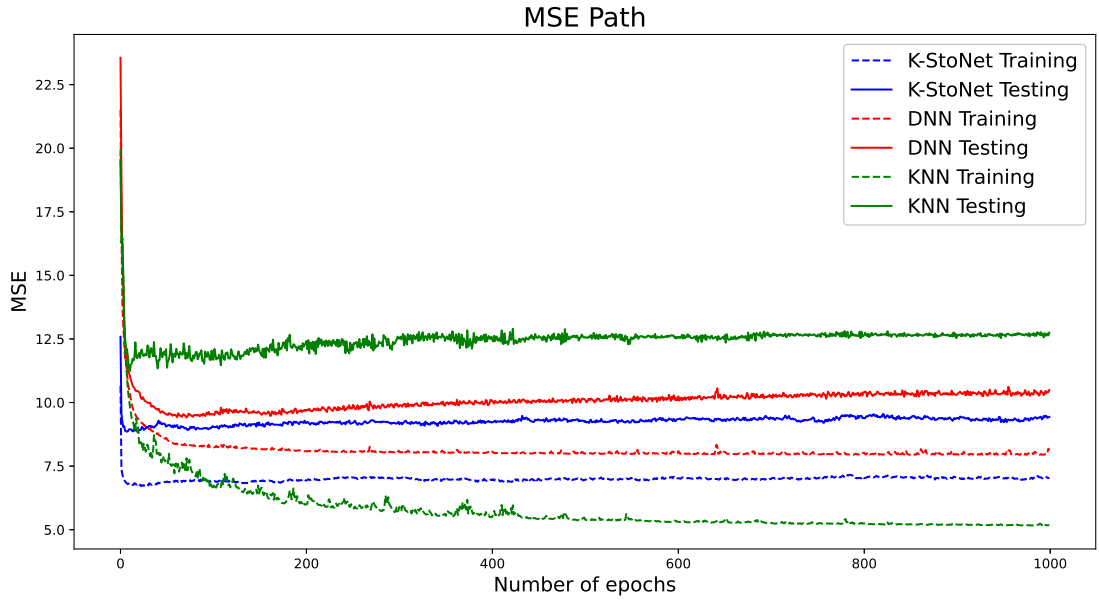} \\ 
\includegraphics[height=3.5in,width=1.0\textwidth]{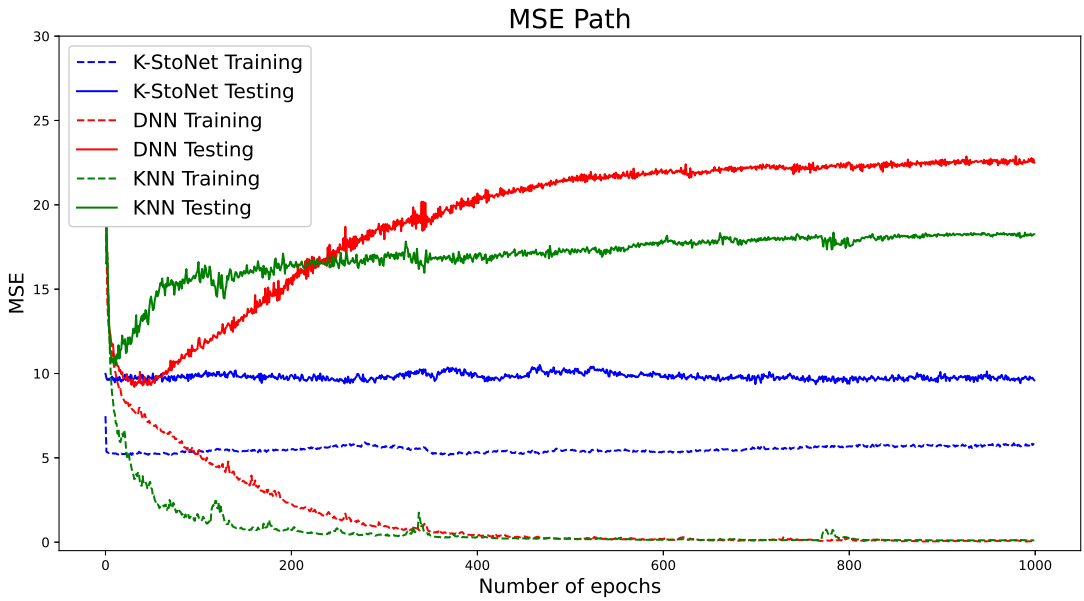}
\end{tabular}
\caption{MSE paths produced by K-StoNets and DNNs: (upper) one-hidden-layer networks; 
(lower) three-hidden-layer networks.}
\label{Measurement_Error}
\end{figure}

\textcolor{black}{
Finally, we explored the sparsity of the SVR layer by varying the value of $\varepsilon$ defined in (\ref{SVRdist}). 
Table \ref{SV_table} shows the number of support vectors selected by the two K-StoNets, together with their training and test errors, at different values of $\varepsilon$.
It implies that the sparsity of K-StoNet can be controlled by $\varepsilon$, a larger $\varepsilon$ 
leading to less support vectors.
However, the two K-StoNet models show different sensitivities to $\varepsilon$. 
The 3-hidden layer K-StoNet has a higher representation power and is more flexible; it can achieve relatively low training error with a large number of connections, and is more sensitive to $\epsilon$. When $\epsilon$ increases from 0.01 to 0.09, it changes from an overfitted model to an underfitted model. Correspondingly, the training MSE increases, while the test MSE decreases in the beginning and then starts to increase.
In contrast, for the 1-hidden layer K-StoNet, its representation power is limited, and it is less flexible and  thus less sensitive to $\epsilon$. It led to about the same models with different choices of $\epsilon$ (with similar training and test errors). The training and test errors varied slightly with $\epsilon$, as different sets of support vectors were used for different choices of $\epsilon$. In general, the set of support vectors used for a large  $\epsilon$ is not nested to that for a small $\epsilon$. Therefore, a smaller $\epsilon$ does not necessarily lead to a smaller training error. }

 \begin{table}[htbp]
\caption{\textcolor{black}{Performance of the K-StoNet model with different values of $\varepsilon$, where the model was evaluated at the last iteration, \#SV represents the average number of support vectors selected by the SVRs at the first hidden layer, and the number in the parentheses represents the standard deviation of the average. }
}
\label{SV_table}
\begin{center}
\begin{tabular}{cccccccc}
 \toprule
  & \multicolumn{3}{c}{1-Hidden Layer} & & 
  \multicolumn{3}{c}{3-Hidden Layer} \\ \cline{2-4} \cline{6-8}
 $\varepsilon$ & \#SV & Train\_MSE & Test\_MSE  & & \#SV & Train\_MSE & Test\_MSE \\ 
 \midrule
 0.01 & 473.6(9.22)  & 6.6786  & 9.0852 & & 409.0(6.419)  & 4.6552  & 10.1387   \\
 
  0.02 & 428.6(18.73)  & 6.7382  & 9.0387 & & 320.2(13.85)  & 4.7257  & 9.9422   \\

  0.03 & 416.2(16.59)  & 6.6917  & 9.0531 & & 235.8(7.94)  & 4.9850  & 9.8192 \\

0.04 & 381.0(29.75)  & 6.7301  & 9.0742  & & 165.4(8.96)  & 5.2578  & 9.3949 \\

 0.05 & 361.4(45.85)  & 6.5158  & 9.3358  & & 113.2(7.83)  & 5.6594  & 9.0195  \\

0.06 & 342.6(28.35)  & 6.5041  & 9.1498  & & 69.6(4.50)  & 6.1869  & 8.7551   \\

 0.07 & 347.2(13.39)  & 6.4878  & 9.1893  & & 37.8(3.60)  & 6.8884  & 9.0565  \\

0.08 & 313.6(38.52)  & 6.5083  & 9.0639 & & 17.6(3.61)  & 7.8616  & 9.5882  \\

 0.09 & 317.0(14.59)  & 6.4130  & 9.0330  & & 16.2(16.45)  & 9.1655  & 9.9940  \\

 0.1 & 290.2(40.45)  & 6.4676  & 9.0308 & &  8.2(13.47)  & 10.5946  & 10.7611   \\
 \bottomrule
\end{tabular}
\vspace{-0.25in}
\end{center}
\end{table}

\section{Real Data Examples} \label{numericalII}

This section shows that a small K-StoNet can work well for a variety of problems. The example in Section \ref{realex1} has a high dimension, which represents typical problems that support vector machine/regression works on. 
The examples in Sections \ref{realex2} and \ref{realex3} have large training sample sizes, which represent typical problems that the DNN works on. 
The examples in Section \ref{realex4} represent more real world problems, with which we explore the prediction performance of K-StoNet.

\subsection{QSAR Androgen Receptor} \label{realex1}

The QSAR androgen receptor dataset is available at the UCI machine learning repository, which consists of 1024 binary attributes (molecular fingerprints) used to classify 1687 chemicals into 2 classes (binder to androgen receptor/positive, non-binder to androgen receptor /negative), i.e. $n = 1607, p = 1024$. The experiment was done in a 5-fold cross-validation. 

In each fold of the experiment, we modeled the data by a K-StoNet with one hidden layer and 5 hidden units, and trained the model by IRO for 40 epochs. The prediction was computed by averaging over the models generated in the last 10 epochs. For comparison,  support vector machine (SVM), KNN and DNN were applied to this example. For SVM, we employed the RBF kernel with $C=1$. \textcolor{black}{For KNN, we used the same structure as K-StoNet.}
For DNN, we tried two network structures, 1024-5-1 and 1024-10-5-1, which
 are called DNN\_one\_layer and DNN\_two\_layer, respectively. 
Each of the KNN and DNN models were
trained by SGD for 1000 epochs with a mini-batch size of 32 and a constant learning rate of 0.001. The weights of the DNNs were subject to the LASSO penalty with the regularization parameter $\lambda=1e-4$.

\begin{figure}[htbp]
\centering
\includegraphics[height=3.0in,width=0.45\textwidth]{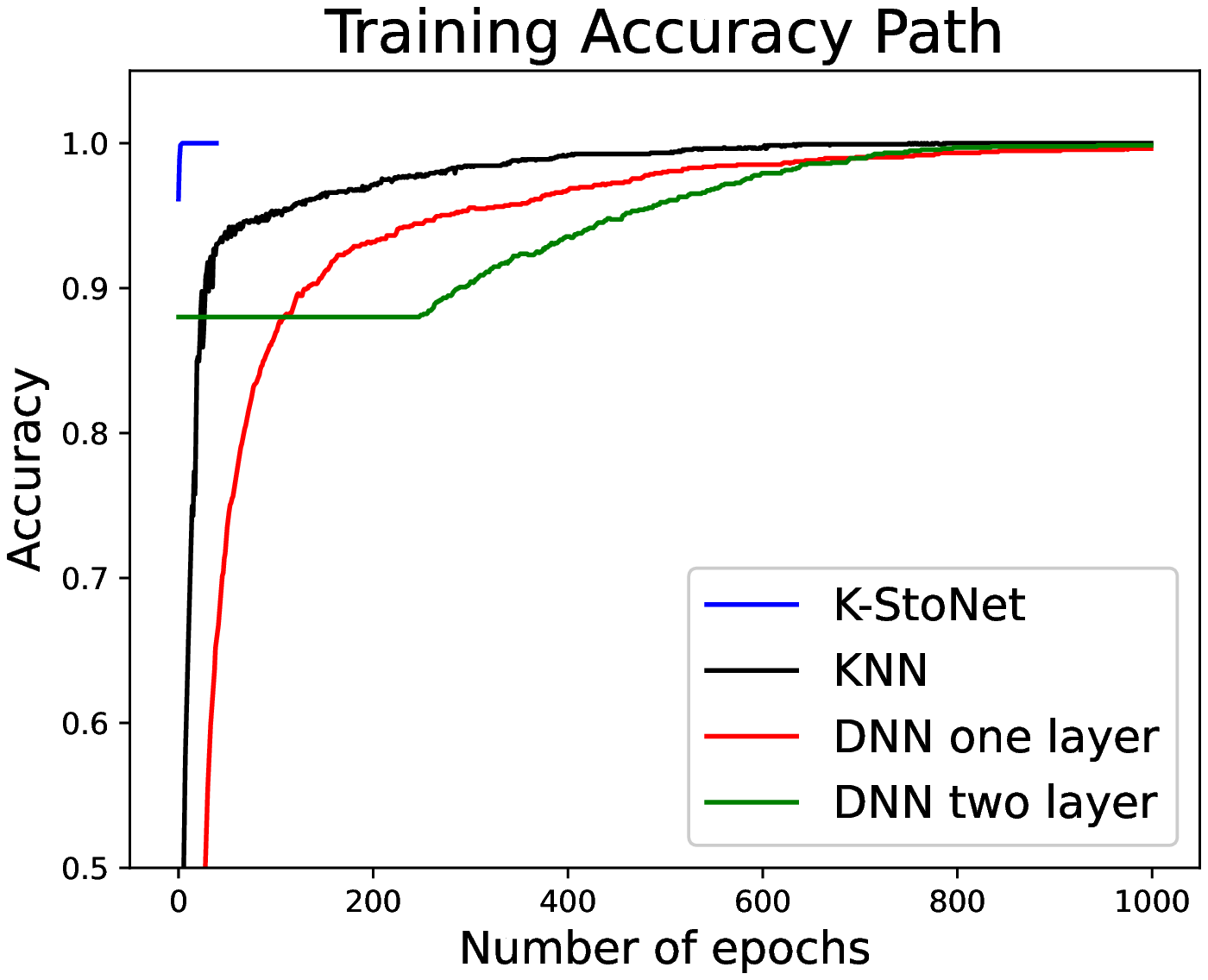}
\includegraphics[height=3.0in,width=0.45\textwidth]{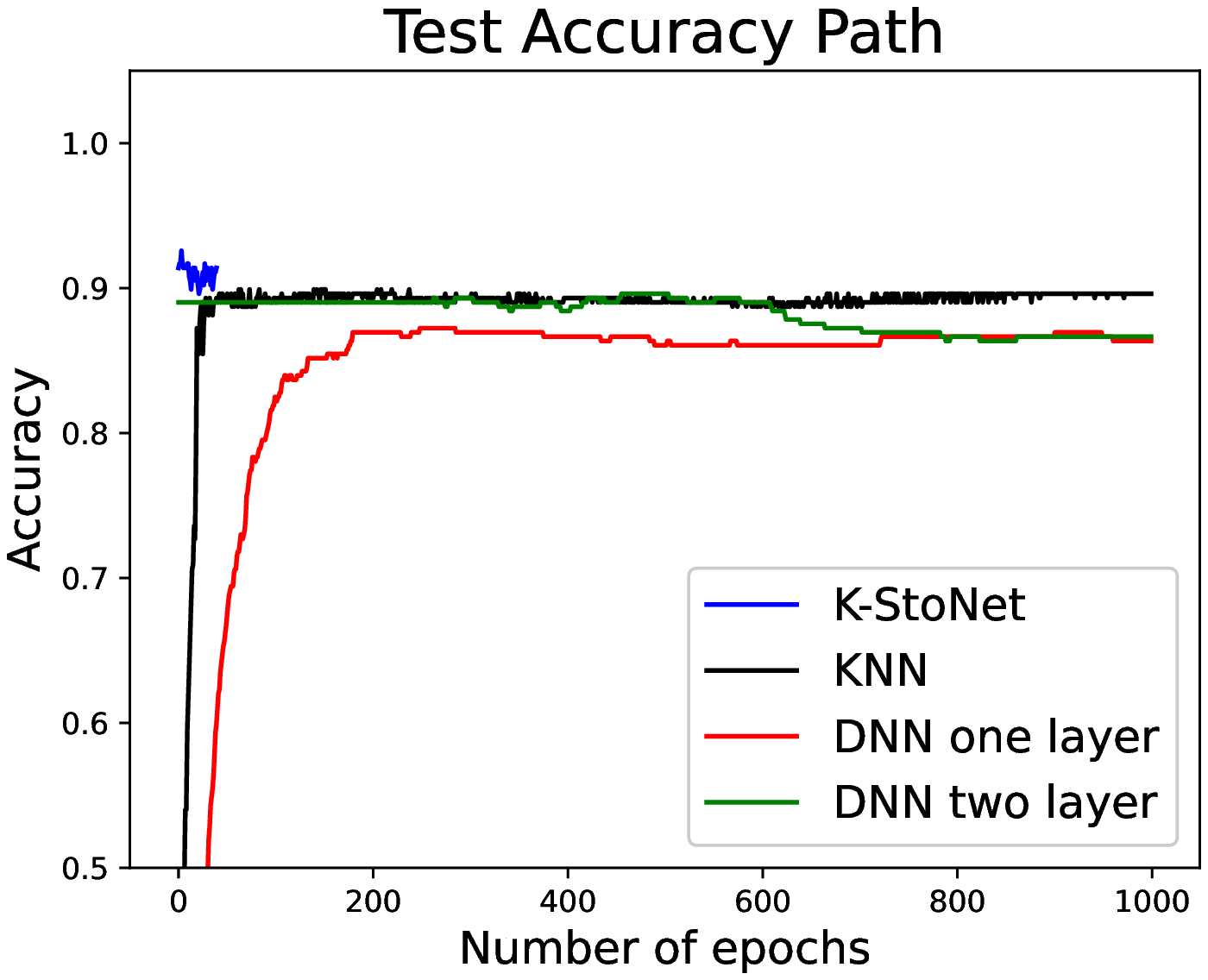}
\caption{Training and prediction accuracy paths (along with epochs) produced by K-StoNet, KNN and DNN in one fold of the cross-validation experiment for the QSAR androgen receptor data. }
\label{QSAR}
\end{figure}
 
Figure \ref{QSAR} compares the training and prediction accuracy paths produced by K-StoNet, KNN and DNN in one fold of the experiment.
Table \ref{QSAR_result} summarizes the training and prediction accuracy produced by K-StoNet, SVM, KNN and DNN over the five folds. In summary,
 K-StoNet converges very fast to the global optimum with the training accuracy close to 1, and is less bothered  by the over-fitting issue.  
In contrast, the KNN and DNN models took more epochs to converge and predicted less accurately than K-StoNet. SVM is inferior to K-StoNet in both training and prediction.

\textcolor{black}{Since each iteration of the IRO algorithm involves imputing latent variables and solving a series of SVR/linear regressions, it is more expensive than a single gradient update step used in DNN training. To compare their computational efficiency, we include an accuracy versus time plot in Figure \ref{qsar_time}, which indicates that 
K-StoNet took less computational time than KNN and 
DNN to achieve the same training/prediction accuracy. }

\begin{figure}
\centering
\includegraphics[height=3.0in,width=0.45\textwidth]{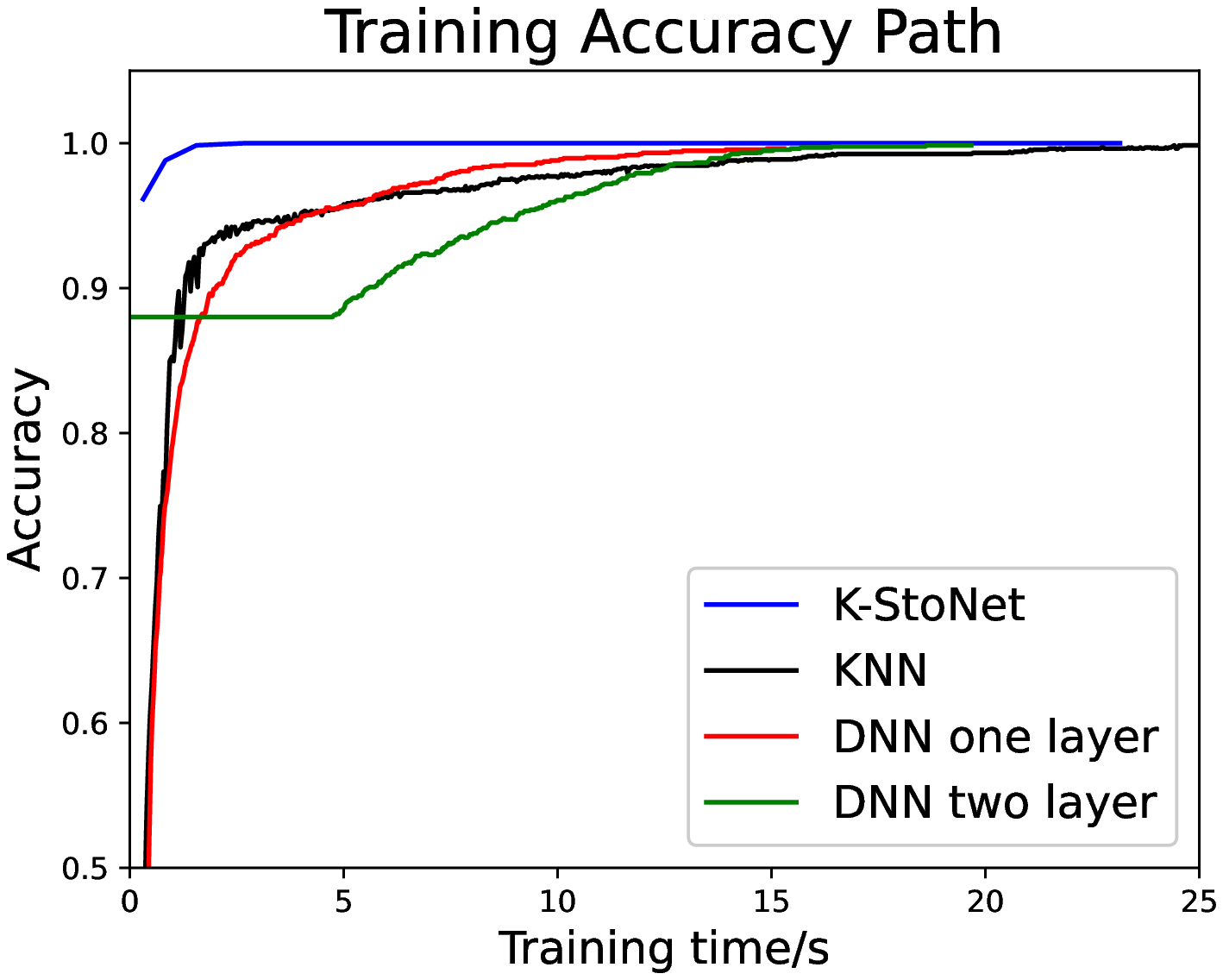}
\includegraphics[height=3.0in,width=0.45\textwidth]{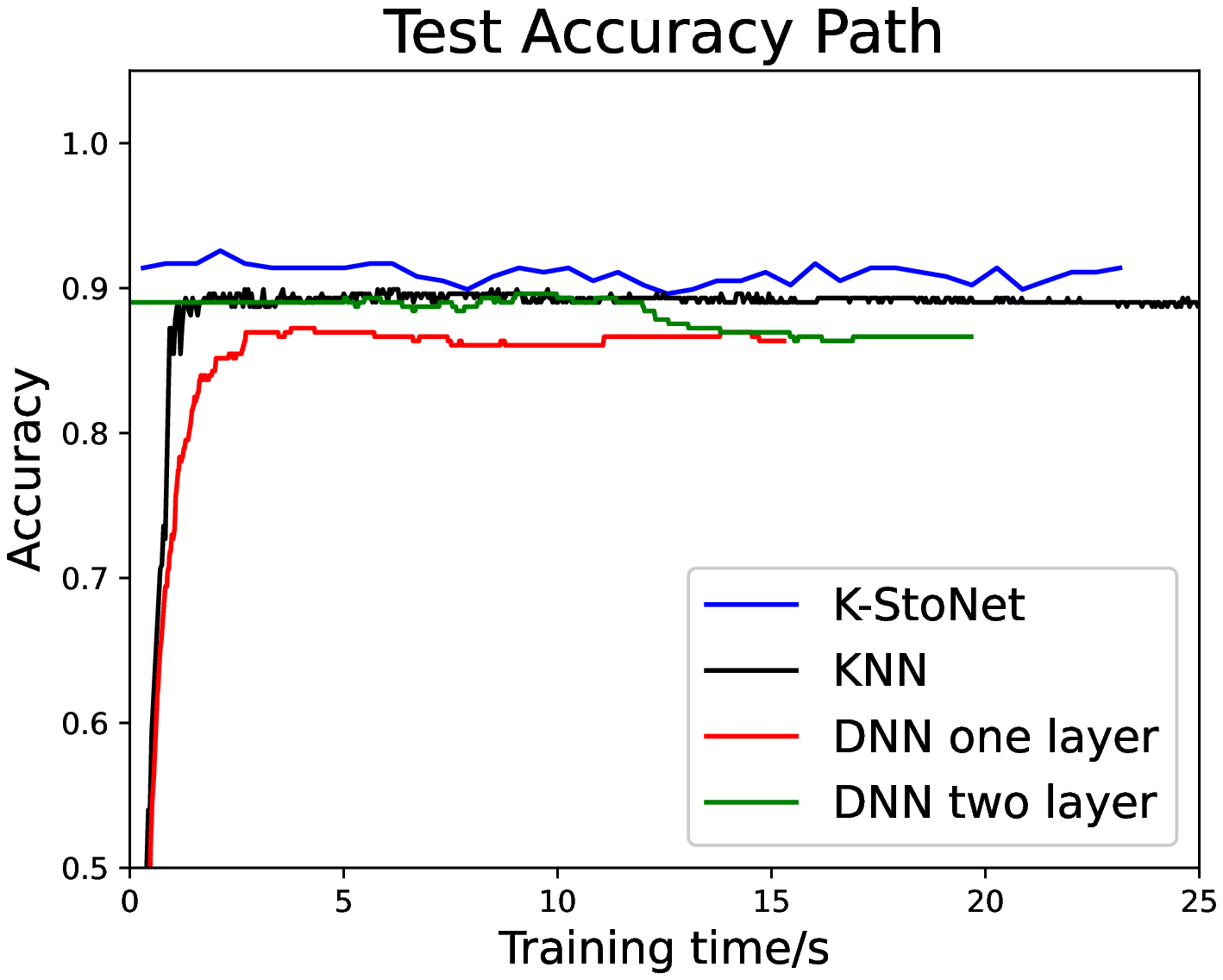}
\caption{Training and prediction accuracy paths (along with computational time) produced by K-StoNets, KNN and DNNs in one fold of the cross-validation experiment for the QSAR androgen receptor data.}
\label{qsar_time}
\end{figure}

\begin{table}[htbp]
\caption{Training and prediction accuracy(\%) for QSAR androgen receptor data, where ``T'' and ``P'' denote the training and prediction accuracy, respectively. }
\label{QSAR_result}
\begin{center}
\vspace{-0.15in}
\adjustbox{max width=1.0\linewidth}{
\begin{tabular}{cccccccc}
 \toprule
Method  & & Split 1  & Split 2 & Split 3 & Split 4 & Split 5  &  Average \\
  \midrule
         &T&   99.93 & 99.85 & 99.85 & 100 & 100 &       99.926 \\   
\raisebox{1.5ex}{K-StoNet} &P&   88.43 & 93.47 & 90.50 & 90.50 & 91.99 & 90.978  \\ \hdashline
    &T& 97.11 & 96.89 & 97.11 & 97.11 & 97.19 & 97.082 \\
\raisebox{1.5ex}{SVM} &P& 89.32 & 90.80 & 87.83 & 89.02 & 92.28 & 89.850 \\ \hdashline
&T& 99.93 & 100 & 98.44 & 99.93 & 99.93 & 99.646 \\
\raisebox{1.5ex}{KNN} &P& 88.72 & 93.17 & 89.32 & 90.50 & 91.10 & 90.562 \\ \hdashline
&T& 99.70 & 99.63 & 99.85 & 99.78 & 99.85 & 99.762 \\
\raisebox{1.5ex}{DNN\_one\_layer}  &P& 85.16 & 88.72 & 89.32 & 86.35 & 88.13  & 87.536 \\ \hdashline
 &T& 99.93 & 99.93 & 99.93 & 100 & 100 & 99.958 \\
\raisebox{1.5ex}{DNN\_two\_layer} &P& 86.94 & 88.13 & 88.13 & 86.05 & 88.72 & 87.596 \\ 
 \bottomrule
\end{tabular}}
\end{center}
\vspace{-0.2in}
\end{table}

\subsection{MNIST Data} \label{realex2}

 The MNIST \citep{lecun1998gradient} is a benchmark dataset in machine learning. It
consists of 60,000 images for training and 10,000 images for testing. We modeled the data by a K-StoNet with  one hidden layer, 20 hidden units, and the softplus activation function, We trained the model by IRO for 6 epochs.
For comparison, we trained a standard LeNet-300-100 model \cite{lecun1998gradient} by Adam \cite{kingma2014adam} with default parameters for 300 epochs, a constant learning rate of 0.001, and a mini-batch size of 128. Figure \ref{mnist_path} shows the training paths of two models. Both models can achieve 100\% training accuracy. LeNet-300-100 achieved 98.38\% test accuracy, while K-StoNet achieved 98.87\% test accuracy (at the 3rd iteration) without data augmentation being used in training!

\begin{figure}[htbp]
 \centering
  \includegraphics[height=3.0in,width=1.0\textwidth]{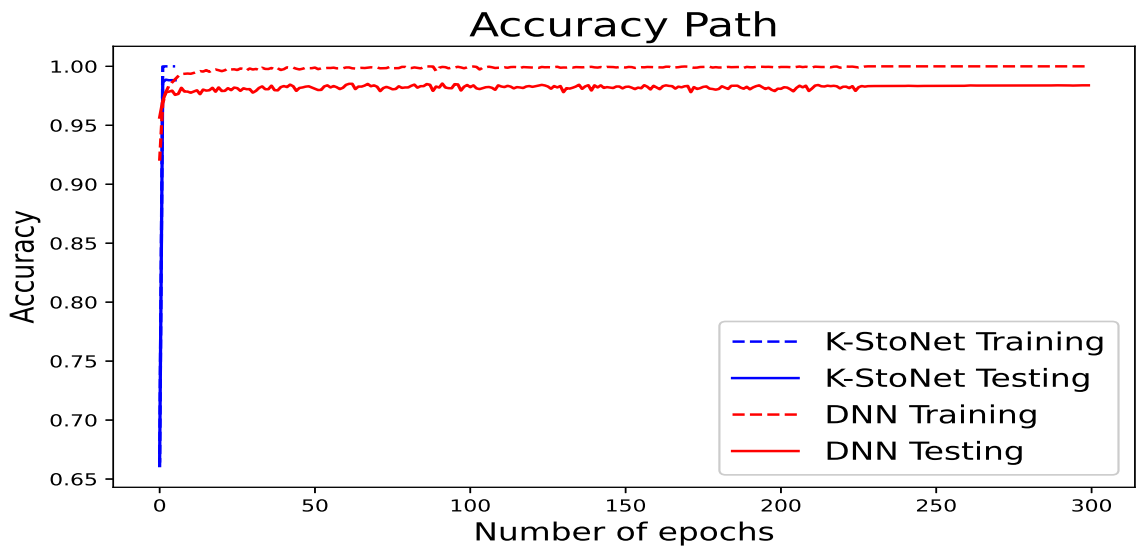}
 \caption{Training and test accuracy versus epochs produced by K-StoNet and DNN (LeNet-300-100) for the  MNIST data, where K-StoNet achieved a prediction accuracy of 98.87\%, and LeNet-300-100 achieved a prediction accuracy of 98.38\%.}
 \label{mnist_path}
\end{figure}

\subsection{CoverType Data} \label{realex3}

\textcolor{black}{The CoverType data is available at the UCI machine learning repository. It consisted of $n=581,012$ samples with $p=54$ attributes, which were collected for classification of forest cover types from cartographic variables. This dataset has an extremely large sample size, which represents a typical problem that the DNN works on. We used half of the samples for training and the other half for testing. The experiments were repeated for thee times.} 

\textcolor{black}{We modeled the data by a K-StoNet with one hidden layer and 50 hidden units, and trained the model by the IRO algorithm for 2 epochs. For comparison, we also modeled the data by a 2-hidden-layer DNN with 1000 nodes on the first hidden layer and 50 nodes on the second hidden layer. We trained the DNN model by SGD with momentum, where  a mini-batch size of 500, a constant learning rate of 0.01 and a momentum decay factor of 0.9 were used. The numerical results were summarized in Table \ref{CoverType_result_update}, which indicates that  K-StoNet outperforms DNN in both training and prediction for this example. }

\begin{table}[htbp]
\caption{Training and prediction accuracy(\%) for CoverType, where ``T'' and ``P'' denote the training and prediction.  }
\label{CoverType_result_update}
\vspace{-0.15in}
\begin{center}
\adjustbox{max width=1.0\linewidth}{
\begin{tabular}{cccccccc}
 \toprule
Method  & & Run 1  & Run 2 & Run 3  &  Average \\
  \midrule
         &T& 99.22 & 99.32   & 99.33  &   99.29    \\   
\raisebox{1.5ex}{K-StoNet} &P& 94.21 &94.23  & 94.26  &  94.23 \\ \hdashline
&T& 98.20 & 98.11 & 98.07 & 98.13   \\
\raisebox{1.5ex}{DNN}  &P& 94.11 & 94.06 & 94.04 &   94.07  \\ 
 \bottomrule
\end{tabular}}
\end{center}
\vspace{-0.2in}
\end{table}

\subsection{More UCI Datasets} \label{realex4}

As shown by the above examples, K-StoNet can converge in only a few epochs and is less bothered by overfitting, and its prediction accuracy is typically similar or better than the best one that DNN achieved. In order to achieve the best prediction accuracy, DNN often needs to be trained with tricks such as early stopping or Dropout \cite{srivastava2014dropout}, which lack the theoretical guarantee for the down-stream statistical inference. In contrast, K-StoNet possesses the theoretical guarantee to asymptotically converge to the global optimum and enables the prediction uncertainty easily assessed (see Section \ref{uncertainsection}). This subsection compares 
K-StoNet with Dropout in prediction on more real world examples, 10 datasets taken at the UCI machine learning repository.  

\textcolor{black}{ 
 Following the setting of \cite{MCdropout2016}, we randomly split each dataset into training and test sets with 90\% and 10\% of its samples, respectively. The random split was repeated for 20 times. The average prediction accuracy and its standard deviation were reported.  As in \cite{MCdropout2016}, for the largest two datasets {\it Protein Structure} and {\it Year Prediction MSD}, the random splitting was done five times and one time, respectively. The baseline results were taken from \cite{MCdropout2016} and  \cite{hernandez2015probabilistic}. The neural network model used there had one hidden layer and 50 hidden units for all datasets except for the largest two. For the largest two datasets, the neural network model had one hidden layer and 100 hidden units. The K-StoNet model we used had one hidden layer with 5 hidden units and the softplus activation function for all datasets except for the largest one. For the largest dataset, K-StoNet had one hidden layer with 50 hidden units. Other parameter settings were given in the Appendix. The results were summarized in Table \ref{uci_result}, which indicates that K-StoNet 
generally outperforms Dropout in prediction. }

\textcolor{black}{For a thorough comparison, the KNN has also been implemented for the UCI datasets except for two large ones, ``Protein Structure'' and ``Year Prediction MSD''. For these two datasets, 
the training sample size $n$ is too large, making the gram matrix hard to handle. For the same reason, it is not included in the comparisons for the MNIST and CoverType data examples, either.
For the KNN, the use of minibatch data is not very helpful when $n$ is large, as there are still $n$ kernels $(K(\bx^{(1)}, \bx^{*}), K(\bx^{(2)},\bx^{*}), \ldots, K(\bx^{(n)},\bx^{*}))$ we need to evaluate for each sample $\bx^{*}$ in the minibatch. 
For the other datasets, the KNN was run with the same network structure as for the K-StoNet. 
The detailed parameter settings were given in the Appendix. 
The results are summarized in Table \ref{uci_result}, which indicates that 
the KNN is generally inferior to K-StoNet in prediction. 
}

\begin{table}[htbp]
\caption{Average test RMSE (and its standard error) by variational inference (VI, \cite{graves2011practical}), probabilistic back-propagation (PBP, \cite{hernandez2015probabilistic}), dropout (Dropout, \cite{MCdropout2016}), SGD via back-propagation (BP),  and KNN, where $N$ denotes the dataset size and $p$ denotes the input dimension. \textcolor{black}{For each dataset, the boldfaced values are the best result or the second best result if it is insignificantly different from the best one according to a $t$-test with a significance level of 0.05.} }
\label{uci_result}
\vspace{-0.2in}
\adjustbox{max width=\textwidth}{
\begin{tabular}{@{}l@{\hspace{4mm}}r@{\hspace{3mm}}r@{\hspace{4mm}}r@{\hspace{2mm}}r@{\hspace{2mm}}r@{\hspace{4mm}}r@{\hspace{2mm}}r@{\hspace{2mm}}r@{}}
\multicolumn{3}{c}{} & \multicolumn{3}{c}{ } \\ \toprule
\textbf{Dataset} & $N$ & $p$ & 
\multicolumn{1}{c}{{VI}} & 
\multicolumn{1}{c}{{PBP}} & 
\multicolumn{1}{c}{{Dropout}} & 
\multicolumn{1}{c}{{BP}} & 
\multicolumn{1}{c}{{KNN}} & 
\multicolumn{1}{c}{{K-StoNet}}  \\ 
\midrule
Boston Housing & 506 & 13 &  
4.32 \tpm 0.29 &  {3.014 \tpm 0.1800} & \textbf{2.97 \tpm 0.19}  & {3.228 \tpm 0.1951} & 4.196 \tpm 0.069& \textbf{2.987 \tpm 0.0227}\\ 
Concrete Strength & 1,030 & 8 & 
7.19 \tpm 0.12 & 5.667 \tpm 0.0933 & \textbf{5.23 \tpm 0.12} & 5.977 \tpm 0.2207  & 6.962 \tpm 0.062 & \textbf{5.261 \tpm 0.0265}\\ 
Energy Efficiency & 768 & 8 & 
2.65 \tpm 0.08 & 1.804 \tpm 0.0481 & 1.66 \tpm 0.04 & \textbf{ 1.098 \tpm 0.0738} & 1.942 \tpm 0.030 & 1.301 \tpm 0.015\\ 
Kin8nm & 8,192 & 8 & 
0.10 \tpm 0.00 & 0.098 \tpm 0.0007 & 0.10 \tpm 0.00 & 0.091 \tpm 0.0015 & 0.0917 \tpm 0.0002 & \textbf{0.0747 \tpm 0.0003}\\ 
Naval Propulsion & 11,934 & 16 & 
0.01 \tpm 0.00 & 0.006 \tpm 0.0000 & 0.01 \tpm 0.00  & \textbf{0.001 \tpm 0.0001} & 0.0151 \tpm 0.0001 & \textbf{0.00098 \tpm 0.0001}\\ 
Power Plant & 9,568 & 4 & 
4.33 \tpm 0.04 & 4.124 \tpm 0.0345 & \textbf{4.02 \tpm 0.04} & 4.182 \tpm 0.0402 & 4.033 \tpm 0.010  & \textbf{3.952 \tpm 0.003}\\ 
Protein Structure & 45,730 & 9 & 
4.84 \tpm 0.03 & 4.732 \tpm 0.0130 & 4.36 \tpm 0.01 & 4.539 \tpm 0.0288 & na & \textbf{ 3.856 \tpm 0.005} \\ 
Wine Quality Red & 1,599 & 11 & 
0.65 \tpm 0.01 & {0.635 \tpm 0.0079} & \textbf{0.62 \tpm 0.01} & {0.645 \tpm 0.0098} & 0.675 \tpm 0.004 & \textbf{0.6214 \tpm 0.0008} \\ 
Yacht Hydrodynamics & 308 & 6 & 
6.89 \tpm 0.67 & \textbf{1.015 \tpm 0.0542} & 1.11 \tpm 0.09 & 1.182 \tpm 0.1645 & 7.5334 \tpm 0.0893  & \textbf{0.8560\tpm 0.0795}\\ 
Year Prediction MSD & 515,345 & 90 & 
 9.034 \tpm na & 8.879 \tpm na & \textbf{8.849 \tpm na}  & 8.932 \tpm na & na & 8.881 \tpm na\\ 
\bottomrule
\end{tabular} 
}
\end{table}

\section{Prediction Uncertainty Quantification with K-StoNet} \label{uncertainsection}

\subsection{A Recursive Formula for Uncertainty Quantification}

  The prediction uncertainty of the K-StoNet can be easily assessed with the variance decomposition formula (as known as Eve's law) based on the asymptotic normality theory. More precisely, we can first calculate the variance for the output of the first hidden layer based on the existing theory of SVR \citep{GaoBayesSVM2002},
  then calculate the variance for the output of the second hidden layer based on Eve's law and the theory of linear models, and  
 continue this process till the output layer is reached. For the case that normal regression was done at layer $h+1$ and no penalties was used in solving the optimization (\ref{optsolver2}), the calculation is detailed as follows.
  
 Let $\BY_i^{(t)} \in \mathbb{R}^{n\times m_i}$, $i=1,2,\ldots,h$, denote the matrices of latent variables imputed 
 at iteration $t$, which leads to the updated parameter estimate $\btheta_n^{(t)}$. Let $\bz$ denote a test sample. For each layer $i\in\{1,2,\ldots,h+1\}$, let $\bZ_i^{(t)} \in \mathbb{R}^{m_i}$ denote the output of the KNN (with the parameter $\btheta_n^{(t)}$) at layer $i$;  and let $\bmu_i^{(t)}$ and $\Sigma_i^{(t)}$ denote the mean and covariance matrix of $\bZ_i^{(t)}$, respectively. 
  Assume that $\bZ_i^{(t)}$'s are all multivariate Gaussian, which will be justified below.
  Then, for any layer $i \in \{2,\ldots,h+1\}$, by Eve's law,  we have
  \[
  \small
  \begin{split}
  \Sigma_{i}^{(t)} & =\mbE (\var(\bZ_{i}^{(t)}|\bZ_{i-1}^{(t)})) + \var(\mbE(\bZ_{i}^{(t)}|\bZ_{i-1}^{(t)})) \\
  &= \mbE\big\{ (\psi(\bZ_{i-1}^{(t)}))^T [(\psi(\BY_{i-1}^{(t)}))^T \psi(\BY_{i-1}^{(t)})]^{-1} \psi(\bZ_{i-1}^{(t)})\big\} 
   \diag\{\sigma_{i,1}^{2(t)},\ldots,\sigma_{i,m_i}^{2(t)} \}+ 
   \var(\tilde{\bw}_{i-1}^{*} \psi(\bZ_{i-1}^{(t)})) \\
   & = \big\{ tr([(\psi(\BY_{i-1}^{(t))})^T \psi(\BY_{i-1}^{(t)})]^{-1} \var(\psi(\bZ_{i-1}^{(t)}))) 
    + (\mbE(\psi(\bZ_{i-1}^{(t)})))^T 
  [(\psi(\BY_{i-1}^{(t)}))^T \psi(\BY_{i-1}^{(t)})]^{-1} 
  (\mbE(\psi(\bZ_{i-1}^{(t)})))
  \big\} \\
  & \quad \times diag\{\sigma_{i,1}^{2(t)},\ldots,\sigma_{i,m_i}^{2(t)}\} + \tilde{\bw}_{i-1}^* \var(\psi(\bZ_{i-1}^{(t)})) (\tilde{\bw}_{i-1}^*)^T,
   \end{split}
   \]
   where $\tilde{\bw}_{i-1}^*=\mbE(\tilde{\bw}_{i-1}^{(t)})$ and $\sigma_{i,j}^{2(t)}$'s are unknown.  
  Let $\bmu_{i-1}^{(t)}=(\mu_{i-1,1}^{(t)}, \ldots, 
  \mu_{i-1,m_{i-1}}^{(t)})^T$ and 
  $D_{\psi'}(\bmu_{i-1}^{(t)})= \diag\{\psi'(\mu_{i-1,1}^{(t)}), \ldots,\psi'(\mu_{i-1,m_{i-1}}^{(t)}) \}$.
  By the first order Taylor expansion, it is easy to derive that
  \[
  \small
  \begin{split}
  \mbE(\psi(\bZ_{i-1}^{(t)})) & \approx \psi(\bmu_{i-1}^{(t)}), \\
  \var(\psi(\bZ_{i-1}^{(t)})) & \approx 
  D_{\psi'}(\bmu_{i-1}^{(t)}) \Sigma_{i-1}^{(t)} 
  D_{\psi'}(\bmu_{i-1}^{(t)}).
  \end{split} 
  \]
 We suggest to estimate $\tilde{\bw}_{i-1}^*$ by $\tilde{\bw}_{i-1}^{(t)}$, estimate  $\bmu_{i-1}^{(t)}$ by $\bZ_{i-1}^{(t)}$, and estimate $\sigma_{i,j}^{2(t)}$ by its OLS estimator from the corresponding multiple regression. This leads to the following recursive formula for covariance estimation:
   \begin{equation} \label{varest}
   \small
   \begin{split}
   \widehat{\Sigma}_{i}^{(t)} & \approx
   \big\{ tr\big([(\psi(\BY_{i-1}^{(t)}))^T \psi(\BY_{i-1}^{(t)})]^{-1} 
   D_{\psi'}(\bZ_{i-1}^{(t)}) \widehat{\Sigma}_{i-1}^{(t)} 
  D_{\psi'}(\bZ_{i-1}^{(t)})  
   \big) \\
   & \quad + (\psi(\bZ_{i-1}^{(t)}))^T 
  [(\psi(\BY_{i-1}^{(t)}))^T \psi(\BY_{i-1}^{(t)})]^{-1} \psi(\bZ_{i-1}^{(t)}) \big\} \diag\{ \hat{\sigma}_{i,1}^{2(t)}, 
  \ldots, \hat{\sigma}_{i,m_{i}}^{2(t)}\} \\
   & \quad +  \tilde{\bw}_{i-1}^{(t)} 
    D_{\psi'}(\bZ_{i-1}^{(t)}) \widehat{\Sigma}_{i-1}^{(t)} 
  D_{\psi'}(\bZ_{i-1}^{(t)})
   (\tilde{\bw}_{i-1}^{(t)})^T,  
   \end{split}
   \end{equation}
for $i=2,3,\ldots,h+1$.   
 For the SVR layer, the asymptotic normality of 
 $\bZ_{1}^{(t)}$ can be justified based on 
 \cite{GaoBayesSVM2002}, which gives a Bayesian interpretation to the classical SVR with an $\varepsilon$-intensive loss function. 
  Let $f(\bx)=\bbeta^T \phi(\bx)+\bb$ denote a SVR function, and let $\{(\bx^{(1)},y^{(1)}),\ldots,(\bx^{(n)},y^{(n)})\}$ denote training samples. 
  In \cite{GaoBayesSVM2002}, the authors treated $(f(\bx^{(1)}),f(\bx^{(2)}),\ldots,f(\bx^{(n)}))$ as a random vector subject to a functional Gaussian process prior, and showed that the posterior of $f(\bz)$ can be approximated by a Gaussian distribution with mean $f^*(\bz)$ and variance 
  \begin{equation} \label{SVRvareq}
 \sigma_{f}^2(\bz)=K_{\bz,\bz}-K_{X_M,\bz}^T K_{X_M,X_M}^{-1} K_{X_M,\bz},
 \end{equation}
 where $f^*(\cdot)$ denotes the optimal regression function fitted by SVR, $X_M=\{\bx^{(s)}: |y^{(s)}-f^{(t)}(\bx^{(s)})|=\varepsilon\}$ denotes the set of marginal vectors, and $K_{A,B}$ denotes a  kernel matrix with elements formed by variables in  $A$ versus variables in $B$. By this result, conditioned on the training samples, 
 $\bZ_1^{(t)}$ is approximately Gaussian with the covariance matrix given by 
 $\diag\{\sigma_{f_1}^{2(t)}(\bz), \ldots, 
 \sigma_{f_{m_1}}^{2(t)}(\bz)\}$.

For the output and other hidden layers, we can restrict $m_i \prec \sqrt{n}$ for each layer  $i\in\{1,2,\ldots,h\}$.
Then, by the theory of \cite{Portnoy1988} which allows the dimension of the parameters diverging with the training sample size, $\bZ_{i+1}^{(t)}$ is asymptotically Gaussian.

Let ${Y}_{\bz}$ denote the unknown true observation at the test point $\bz$, and let
$\hat{\xi}^{(t)}(\bz)=Z_{h+1}^{(t)}$ denote its K-StoNet prediction with the parameters $\btheta_n^{(t)}$. Then the variance of ${Y}_{\bz}-\hat{\xi}(\bz)$ can be approximated by
\begin{equation} \label{finalvar}
\small
\widehat{\var}(Y_{\bz}-\hat{\xi}(\bz))=\frac{1}{n} \sum_{i=1}^n (y^{(i)}-\hat{\xi}^{(t)}(\bx^{(i)}))(y^{(i)}-\hat{\xi}^{(t)}(\bx^{(i)}))^T + \widehat{\Sigma}_{h+1}^{(t)},
\end{equation}
based on which the prediction interval for $Y_z$ can be constructed at a desired confidence level.
Further, by Theorem \ref{thm:3}, a more accurate confidence interval can be obtained by averaging over those obtained at different iterations.

 The above procedure can be easily extended to the probit/logistic regression (via Wald/endpoint transformation) and the case that an amenable penalty is used in solving (\ref{optsolver2}). Refer to \cite{Loh2017Ann} for asymptotic normality of the regularized estimators. 
 
\subsection{A Numerical Example} \label{uncertainty2num}

To illustrate the above procedure, we generated 100 training datasets and one test dataset as in Section \ref{SimulationII}. Each dataset consisted of 500 samples. Again, we modeled the data using a K-StoNet with one hidden layer and 5 hidden units.
For each training dataset, we trained the model by IRO  for 50 epochs. 
 For each test point $\bz$, a 95\% prediction interval was constructed based on each training dataset according to the prediction variance calculated in (\ref{finalvar}) at the last epoch of the run, and the coverage rate was calculated by averaging over the coverage status (0 or 1) of the 100 prediction intervals. Further, we averaged the coverage rates over 500 test points, which produced a mean coverage rate of 93.812\% with standard deviation 0.781\%. It is very close to the nominal level 95\%!
 Figure \ref{confidence_interval} shows the prediction intervals at 20 test points, which were obtained  
 at the last epoch of an IRO run for a training dataset.

\begin{figure}[htbp]
\centering
\includegraphics[height=3.0in,width=1.0\linewidth]{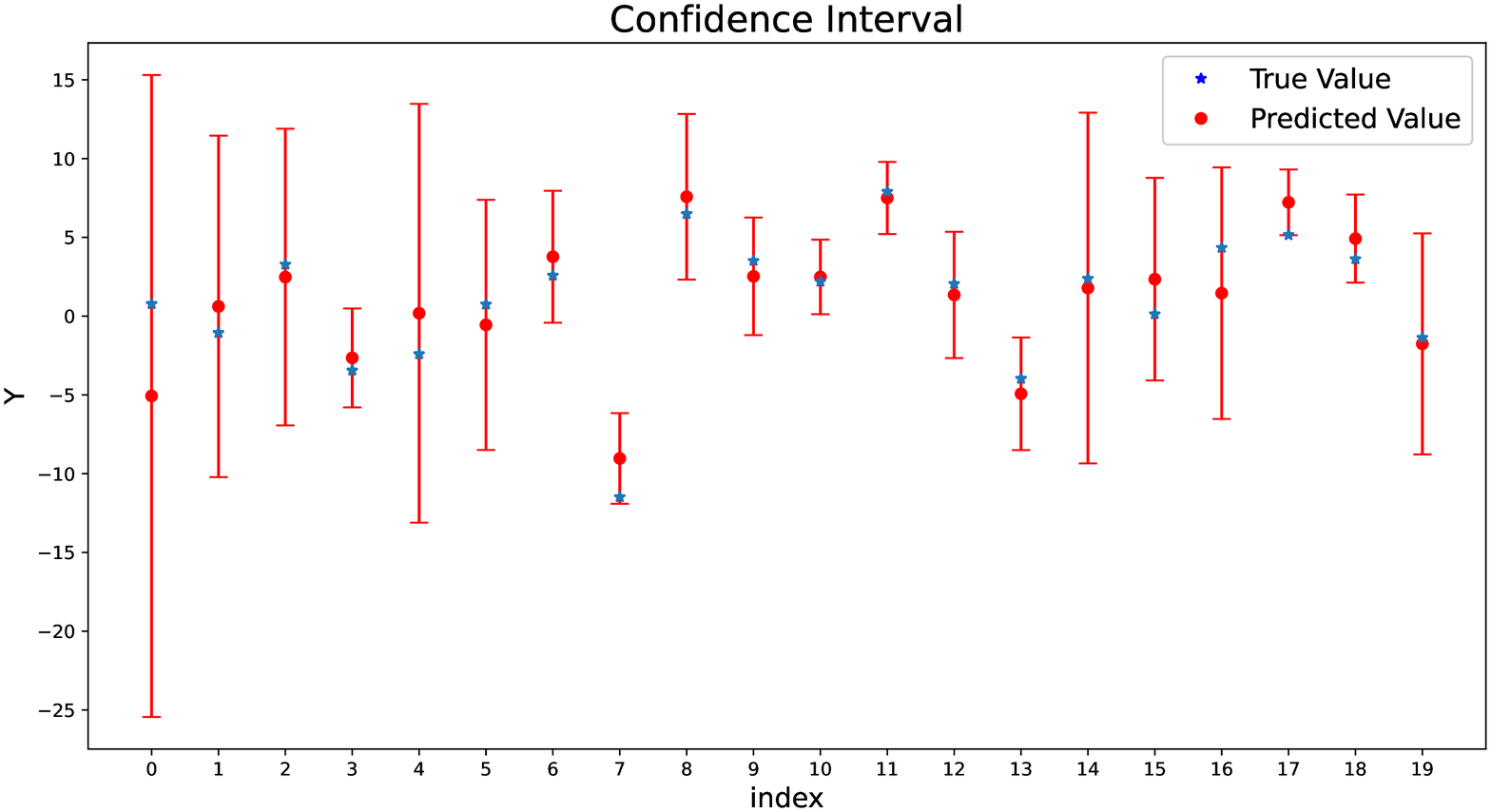}
\caption{95\% prediction intervals produced by K-StoNet for 20 test points, where the x-axis indexes the test points, 
 the y-axis represents the response value, and the blue star represents the true observation.}
\label{confidence_interval}
\end{figure}

For each test point, we have also constructed a prediction interval based on each training dataset 
by averaging those obtained at the last 25 epochs. 
As a result, the mean coverage rate over the 500 test points was improved to 94.026\% with standard deviation 0.771\%.

\section{Discussion} \label{discussionsection}

We have proposed K-StoNet as a new neural network model for machine learning. 
The K-StoNet incorporates SVR as the first hidden layer and reformulates the neural network as a latent variable model. The former maps the input variable into an infinite dimensional feature space via the RBF kernel, ensuring  absence of local minima on the loss surface of the resulting neural network. The latter breaks the high-dimensional nonconvex neural network training problem into a series of lower-dimensional convex optimization problems.   In addition, the use of kernel partially addresses the over-parameterization issue suffered by the DNN; it enables  a smaller network to be used, while ensuring the universal approximation capability.  The K-StoNet can be easily trained using the IRO algorithm. Compared to DNN, K-StoNet avoids local traps, and enables the prediction uncertainty easily assessed. Compared to SVR, K-StoNet has better approximation capability due to the added hidden units. Our numerical results indicate its superiority over SVR and DNN in both training and prediction.

 As an important ingredient of K-StoNet, StoNet is itself of interest. Under the framework of StoNet, the existing statistical theory for SVR and high-dimensional generalized linear models can be easily incorporated into the development of deep learning.
 
 \textcolor{black}{As another important gradient of K-StoNet, kernel has long been studied in machine learning as a function approximation tool. It is known that, with ``kernel trick'', kernel methods enable a classifier/regression to learn a complex decision boundary with only a small number of parameters.
To enhance their flexibility, some researchers proposed the so-called deep kernel learning methods, where one kernel function is repeatedly concatenated with another kernel or nonlinear function, see e.g. \cite{ChoSaul2009,zhuang2011,StroblV2013,mairal2014convolutional, Rebai2016,wilsonXing2016,Bohn2019}. Under some conditions, \cite{StroblV2013} showed that the upper bound of generalization error for deep multiple kernels can be significantly lower than that for the DNNs. However, unlike shallow kernel methods such as SVM and SVR \cite{VapnikL1963,VapnikC1964,Vapnik2000}, 
coefficient estimation for deep kernels is not convex any more. Estimating coefficients of the inner layer kernel can be highly nonlinear and becomes more complicated for a larger number of layers \cite{Bohn2019}. By introducing latent variables, this work provides an effective way to resolve the computational challenge suffered by deep kernel learning. K-StoNet is essentially a deep kernel learning method.}

The K-StoNet can be further extended in many different ways. \textcolor{black}{Instead of relying on a SVR solver, K-StoNet can be implemented as a StoNet with the gram matrix being treated 
as input data. In this case, although a large-scale gram matrix needs to be handled when the training sample size is large,  different kernels can be adopted for different tasks. For example, one might employ the convolutional kernel developed in  \cite{mairal2014convolutional} for computer vision problems. }
As discussed in Section \ref{IROsect}, the regression in the output and other hidden layers can also be regularized by different amenable penalties \cite{loh2017support}, some of which might lead to better selection properties than Lasso.

The K-StoNet has an embarrassingly parallel structure; \textcolor{black}{that is, solving K-StoNet can be broken into to many parallel tasks that can be solved with little or no need for communications. More precisely, the imputation step can be done in parallel for each observation; and the parameter updating step can be done in parallel for each of the regression tasks, including both SVR and Lasso regression.} If both steps are implemented in parallel, the computation can be greatly accelerated.  
 Currently, parallelization has not yet been completed: the imputation step was implemented in PyTorch, where the imputation for each observation was done in a serial manner;
 the parameter updating step was implemented using the package {\it sklearn}\cite{sklearn2011} for the normal or logistic regression 
 and {\it ThunderSVM} \cite{wenthundersvm18} for SVR, where the regression was solved one by one.
 On a machine with Intel(R) Xeon(R) Silver 4110 CPU @ 2.10GHz and Nvidia Tesla P100 GPU, for a dataset with $n=10000$ and $p=54$ and a 1-hidden layer K-StoNet with 5 hidden units, one iteration/epoch cost less than half a minute in the current serial implementation. Therefore, all the examples presented in this paper can be done reasonably fast on the machine. We expect that the computation can be much accelerated with a parallel implementation of K-StoNet.
 
 As mentioned in Section \ref{IROsect}, a scalable SVR solver will accelerate the computation of K-StoNet substantially.
 The scalable SVR can be developed in different ways. 
 For example,  \cite{scalableSVR2015} developed 
 ParitoSVR --- a parallel iterated optimizer for SVR,
 where each machine iteratively solves a small (sub-)problem based only on a subset data and these solutions are then combined to form the solution to the global problem. ParitoSVR is provably convergent to the results obtained from the centralized algorithm, where the optimization has access to the entire data set. Alternatively, one can implement SVR in an 
 incremental manner \cite{OnlineSVR2003,IncrementalSVM2006,XuBoundarySVR2010,GreedySVR2019}, where a SVR is first learned with a subset data and then sequentially updated based on the remaining set of samples. By the property that the decision function of SVR depends on support vectors only, \cite{XuBoundarySVR2010} proposed to use only the boundary vectors in the remaining set of samples.
 By the same property, \cite{GreedySVR2019} proposed 
 a sample selection method for SVR to 
 maximize its validation set accuracy at the minimum number of training examples.  As shown in \cite{XuBoundarySVR2010,GreedySVR2019}, both methods can accelerate the computation of SVR substantially.

\section*{Acknowledgment}
Liang's research is supported in part by the NSF grant DMS-2015498 and the NIH grants R01 GM-117597 and R01 GM-126089.
The authors thank the editor, associate editor, and two referees for their constructive comments which have led to significant improvement of this paper. 
The authors also thank Professor W.H. Wong for his helpful comments on an early version of this paper.

\appendix

\section{Appendix}

\subsection{Proof of Theorem \ref{thm:2}}

\begin{proof}
Since $\Theta$ is compact, it suffices to prove that the consistency holds for any $\btheta\in \Theta$.
For notational simplicity, we rewrite $\sigma_{n,i}$ by $\sigma_i$ in the remaining part of the proof. 

Let $\bYmis=(\bY_1,\bY_2,\ldots,\bY_h)$, where $\bY_i$'s are latent variables as defined in (\ref{Kstoneteq}). Let 
$\tilde{\bY}=(\tilde{\bY}_1,\ldots, \tilde{\bY}_h)$, 
where $\tilde{\bY}_i$'s are calculated by KNN in (\ref{KDNNeq}).
By Taylor expansion, we have 
\begin{equation} \label{Tayloreq}
\small
\begin{split}
\log & \pi(\bY,\bYmis|\bX,\btheta) = \log \pi(\bY, \tilde{\bY}|\bX,\btheta) 
 +\bepsilon^T 
\nabla_{\bYmis} \log \pi(\bY,\tilde{\bY}|\bX,\btheta) + 
O(\|\bepsilon\|^2),
\end{split} 
\end{equation}
where  $\bepsilon=\bYmis-\tilde{\bY}=(\bepsilon_1,\bepsilon_2,\ldots,\bepsilon_h)$, $\nabla_{\bYmis} \log\pi(\bY,\tilde{\bY}|\bX,\btheta)$ is evaluated according to the joint distribution (\ref{jointeq}), and  $\log \pi(\bY,\tilde{\bY}|\bX,\btheta)=\log \pi(\bY|\bX,\btheta)$ is the log-likelihood function of the KNN.

Consider the partial derivative $\nabla_{\bY_i} \log \pi(\bY,\bYmis|\bX,\btheta)$, for whose single component, say $Y_{i}^{(k)}$, the output of neuron $k$ at hidden layer $i\in \{2,3,\ldots,h\}$, we have 
\begin{equation} \label{gradeq}
\small
\begin{split}
\nabla_{Y_i^{(k)}} \log \pi(\bY,\bYmis|\bX,\btheta)
 & = \frac{1}{\sigma_{i+1}^2} \sum_{j=1}^{m_{i+1}} \big({Y}_{i+1}^{(j)}-b_{i+1}^{(j)}-\bw_{i+1}^{(j)} \psi({\bY}_i) \big) 
  w_{i+1}^{(j,k)}  \psi^{\prime}({Y}_{i}^{(k)}) \\
  & \quad - \frac{1}{\sigma_i^2}({Y}_{i}^{(k)}-b_{i}^{(k)}-\bw_i^{(k)} \psi({\bY}_{i-1})), \\
 \end{split}
\end{equation}
where $\bw_{i+1}^{(j)}$ denotes the vector of the weights from neuron $j$ at layer $i+1$ to the neurons at layer $i$, and $w_{i+1}^{(j,k)}$ denotes the weight from neuron $j$ at layer $i+1$ to neuron $k$ at  layer $i$.
For layer $i=1$, the second term of (\ref{gradeq}) will disappear, since the $\epsilon$-intensive loss is a constant around zero.

Since
$Y_{i+1}^{(j)}=b_{i+1}^{(j)}+\bw_{i+1}^{(j)} \psi(\bY_{i})+e_{i+1}^{(j)}$ and 
$\tilde{Y}_{i}^{(k)}=b_{i}^{(k)}+\bw_{i}^{(k)} \psi(\tilde{\bY}_{i-1})$, we have,
for any  $k \in \{1,2,\ldots,m_i\}$, 
 $\nabla_{Y_i^{(k)}} \log \pi(\bY,\tilde{\bY}|\bX,\btheta) =
\frac{1}{\sigma_{i+1}^2} \sum_{j=1}^{m_{i+1}} \big(e_{i+1}^{(j)}+\bw_{i+1}^{(j)} (\psi(\bY_i)- \psi(\tilde{\bY}_i) ) \big) w_{i+1}^{(j,k)}  \psi^{\prime}(\tilde{Y}_{i,k})$ if $i=h$, 
and 0 otherwise.
Then, by Assumption \ref{ass2}-(i)\&(iii),
for any $k \in \{1,2,\ldots,m_i\}$,
\begin{equation} \label{deeq1}
\begin{split}
 & | \nabla_{Y_i^{(k)}} \log \pi(\bY,\tilde{\bY}|\bX,\btheta) | \leq 
\begin{cases} 
\frac{1}{\sigma_{i+1}^2} \big\{ \sum_{j=1}^{m_{i+1}} e_{i+1}^{(j)} w_{i+1}^{(j,k)} \psi'(\tilde{Y}_{i}^{(k)}) + (c'r)^2  m_{i+1} \|\bepsilon_i\|\big\}, & \mbox{$i=h$} \\
0  & \mbox{$i<h$},
\end{cases}
\end{split}
\end{equation}
 where $\bepsilon_i=\bY_i-\tilde{\bY}_i$,  $e_{i+1}^{(j)}$ is the $j$th component of $\be_{i+1}$, $r$ is the upper bound of the weights, and $c'$ is the Lipschitz constant of $\psi(\cdot)$ as well as the upper bound of $\psi'(\cdot)$.

Next, let's figure out the order of $\|\bepsilon_i\|$. The $k$th component of $\bepsilon_i$ is given by  
\begin{equation*} 
\small
Y_{i}^{(k)}-\tilde{Y}_{i}^{(k)}=
\begin{cases}
e_{i}^{(k)}+\bw_i^{(k)} (\psi(\bY_{i-1})-\psi(\tilde{\bY}_{i-1}) ), & \mbox{$i>1$}, \\
e_i^{(k)}, & \mbox{$i=1$}.
\end{cases}
\end{equation*}
Therefore, $\|\bepsilon_1\|=\|\be_1\|$; and for $i=2,3,\ldots,h$, the following inequalities hold:
\begin{equation} \label{deeq3}
\small
\|\bepsilon_i\|  \leq \|\be_i \| + c' r m_i \|\bepsilon_{i-1}\|,  \quad
\|\bepsilon_i\|^2  \leq  2\|\be_i\|^2
+2 (c'r)^2 m_i^2 \|\bepsilon_{i-1}\|^2. 
\end{equation}

Since $\be_i$ and $\be_{i-1}$ are independent, by 
summarizing (\ref{deeq1}) and (\ref{deeq3}), we have 
\begin{equation} \label{coneq111}
\small
\begin{split}
 \int & \bepsilon^T \nabla_{\bYmis} \log \pi(\bY,\tilde{\bY}|\bX,\btheta)   \pi(\bYmis|\bX,\btheta,\bY) d \bYmis 
  \leq
 O\big( \sum_{k=2}^{h+1} \frac{\sigma_{k-1}^2}{\sigma_{h+1}^2} m_{h+1} (\prod_{i=k}^h m_i^2) m_{k-1} \big)=o(1), 
\end{split}
\end{equation}
where the last equality follows from \ref{ass2}-(v). 
Then, by (\ref{Tayloreq}), we have the mean value 
 \[
 \small
 \mbE\left[\log \pi(\bY,\bYmis|\bX,\btheta) - \log \pi(\bY|\bX,\btheta) \right] \to 0, \quad \forall \btheta \in \Theta.
 \]
 Further, it is easy to verify 
 \[
 \int |\bepsilon^T \nabla_{\bYmis} \log \pi(\bY,\tilde{\bY}|\bX,\btheta) |^2  \pi(\bYmis|\bX,\btheta,\bY)  d \bYmis <\infty,
 \]
 which, together with (\ref{Tayloreq}) and (\ref{deeq3}), implies 
 \begin{equation} \label{wllneq}
 \mbE |\log\pi(\bY,\bYmis|\bX,\btheta)-\log \pi(\bY|\bX,\btheta)|^2 <\infty.
 \end{equation}
Therefore, the weak law of large numbers (WLLN) applies, and the proof can be concluded. 
\end{proof}

\subsection{Proof of Lemma \ref{lem2}}

Lemma \ref{lem2} is a direct application of Lemma \ref{lem3} given below. 

 \begin{lemma}\label{lem3} Consider a 
 function $Q(\btheta, \bX_n)$. Suppose that the following 
  conditions are satisfied: (B1) 
  $Q(\btheta,\bX_n)$ is continuous in $\btheta$ and there exists a function 
  $Q^*(\btheta)$, which is continuous in $\btheta$ and 
  uniquely maximized at $\btheta^*$.
  (B2) For any $\epsilon>0$,  $\sup_{\btheta \in \Theta \setminus B(\epsilon)} 
  Q^*(\btheta)$ exists, where 
   $B(\epsilon)=\{\btheta: \|\btheta-\btheta^*\| < \epsilon\}$; 
   Let $\delta=Q^*(\btheta^*)-
    \sup_{\btheta \in \Theta\setminus B(\epsilon)} Q^*(\btheta)$, then $\delta>0$.
  (B3) $\sup_{\btheta \in \Theta} 
       |Q(\btheta, \bX_n)-Q^*(\btheta)| \stackrel{p}{\to} 0$ as $n\to \infty$.
Let $\hat{\btheta}_n=\arg\max_{\btheta\in \Theta}  Q(\btheta, \bX_n)$. 
Then  $\|\hat{\btheta}_n-\btheta^*\|\stackrel{p}{\to} 0$.
\end{lemma}

 \begin{proof} 
 Consider two events (i) $\sup_{\btheta \in \Theta \setminus B(\epsilon)} 
  |Q(\btheta,\bX_n)-Q^*(\btheta)| < \delta/2$, and (ii) 
  $\sup_{\btheta \in B(\epsilon)}
  |Q(\btheta,\bX_n)-Q^*(\btheta)| < \delta/2$. 
 From event (i), we can deduce that for any 
 $\btheta \in \Theta\setminus B(\epsilon)$, 
 $Q(\btheta, \bX_n) < Q^*(\btheta)+\delta/2 \leq Q^*(\btheta^*) -\delta 
  +\delta/2 \leq Q^*(\btheta^*) -\delta/2$. 

 From event (ii), we can deduce that for any $\btheta \in B(\epsilon)$, 
  $Q(\btheta, \bX_n)> Q^*(\btheta) -\delta/2$ and thus
  $Q(\btheta^*, \bX_n)> Q^*(\btheta^*) -\delta/2$. 
  
 If both events hold simultaneously,  then we must have $\hat{\btheta}_n \in B(\epsilon)$  
  as $n \to \infty$.  
 By condition (B3), the probability that both events hold tends to 1.  Therefore,
$P(\mbox{$\hat{\btheta}_n \in B(\epsilon)$}) \to 1$,
which concludes the lemma. 
\end{proof} 

 \subsection{Parameter Settings for K-StoNet} \label{settingsection}
 
 In all computations of this paper except for the CoverType experiments, the RBF kernel
  $k(\bx, \bx') = \exp(-\gamma||\bx - \bx'||_2^2) $ is used, where $\gamma$ is set to the default value $\frac{1}{p \text{Var}(\bx)}$, $p$ is the dimension of $\bx$, and $\text{Var}(\bx)$ is the variance of $\bx$. We have the following default values for the parameters: one hidden layer, 5 hidden units, $C_n=\tilde{C}_{n,1}^{(t)}=10$ for all $t$, $\varepsilon=0.01$, 
 $t_{HMC}=25$, $\alpha=0.1$, $\sigma_{n,2}^2=0.01$, 
 and the learning rate $\epsilon_{t,i}=5e-4$ for all $t$ and $i$. The parameter settings may vary around the default values to achieve better performance for the K-StoNet model.

 \paragraph{Section \ref{SimulationI}} 
 {\it Simulated DNN Data.}
 Network: $C_n=1$ for the SVR layer, $\sigma_{n,2}^2 = 0.001$ for the output layer;
 HMC imputation: $t_{HMC}=25$, $\alpha = 0.1$, and $\epsilon_{t,i}=5e-7$ 
 for all $t$ and $i$;  parameter updating: for all $t$ and $i$, (i) SVR with $\tilde{C}_{n,1}^{(t)}=1$ and $\varepsilon=0.1$, (ii) linear regression with a Lasso penalty and the regularization parameter $\lambda_{n,i}^{(t)}=1e-4$. 
 
 {\it Simulated KNN Data.}   Network: $C_n=5$ for the SVR layer, $\sigma_{n,2}^2 = 0.001$ for the output layer;
 HMC imputation: $t_{HMC}=25$, $\alpha = 0.1$, and $\epsilon_{t,i}=5e-4$ 
 for all $t$ and $i$;  parameter updating: for all $t$ and $i$, (i) SVR with $\tilde{C}_{n,1}^{(t)}=5$ and $\varepsilon=0.01$, (ii) linear regression with a Lasso penalty and the regularization parameter $\lambda_{n,i}^{(t)}=1e-4$.

\paragraph{Section \ref{SimulationII}} {\it Measurement error data.}  For both the one-hidden layer and three-hidden layer K-StoNets, the parameters were set as follows:
Network: $C_n=1$ for the SVR layer, $\sigma_{n,i}^2 = 0.001$ for layers $i=2,\ldots,h$, and 
$\sigma_{n,h+1}^2=0.01$; HMC imputation: $t_{HMC}=25$, $\alpha = 1$, and $\epsilon_{t,i}=5e-5$ 
for all $t$ and $i$;
 parameter updating: for all $t$ and $i$, (i) SVR with $\tilde{C}_{n,1}^{(t)}=1$ and $\varepsilon \in \{0.01,0.02,\ldots,0.1\}$, (ii) linear regression with a Lasso penalty and the regularization parameter $\lambda_{n,i}^{(t)}=1e-4$.

\paragraph{Section \ref{realex1}} {\it QSAR Androgen Receptor.} Network: $C_n=1$ for the SVR layer, $\sigma_{n,2}^2 = 0.001$ for the output layer;
 HMC imputation: $t_{HMC}=25$, $\alpha = 0.1$, and $\epsilon_{t,i}=5e-5$ 
 for all $t$ and $i$; parameter updating: for all $t$ and $i$, (i) SVR with $\tilde{C}_{n,1}^{(t)}=1$ and $\varepsilon=0.1$,  (ii) logistic regression with a Lasso penalty and the regularization parameter $\lambda_{n,i}^{(t)}=1e-4$.

 \paragraph{Section \ref{realex2}} {\it MNIST Data.}
Network: $C_n=10$ for the SVR layer, 
 $\sigma_{n,2}^2 = 1e-9$ for the output layer;
 HMC imputation: $t_{HMC}=25$, $\alpha = 0.1$, and $\epsilon_{t,i}=5e-13$ 
 for all $t$ and $i$; parameter updating: for all $t$ and $i$, (i) SVR with  $\tilde{C}_{n,1}^{(t)}=10$ and $\varepsilon=0.0001$, (ii) multinomial logistic regression with a Lasso penalty and the regularization parameter $\lambda_{n,i}^{(t)}=1e-4$.

 \textcolor{black}{\paragraph{Section \ref{realex3}}{\it CoverType Data.}
Network: $C_n=10$ for the SVR layer, $\sigma_{n,2}^2 = 0.005 $ for the output layer;
 HMC imputation: $t_{HMC}=25$, $\alpha = 0.1$, and $\epsilon_{t,i}=5e-5$ 
 for all $t$ and $i$; parameter updating: for all $t$ and $i$, (i) SVR with $\tilde{C}_{n,1}^{(t)}=10$ and $\varepsilon=0.01$.  (ii) multinomial logistic regression with a Lasso penalty and the regularization parameter $\lambda_{n,i}^{(t)}=1e-4$. This dataset consists of 44 binary features. When applying the RBF kernel $k(\bx, \bx') = \exp(-\gamma||\bx - \bx'||_2^2) $, the default choice $\gamma = \frac{1}{p \text{Var}(\bx)}$ does not work well.  Different values of $\gamma$ were used for different SVRs in the K-StoNet model. Let $\gamma_i$ denote the $\gamma$-value used for the SVR corresponding to the $i$-th hidden unit. We set $\gamma_i = 0.5$ for $1\leq i < 30$, $\gamma_i = 1$ for $30\leq i < 40$, $\gamma_i = 2$ for $40\leq i < 45$, and $\gamma_i = 5$ for $45\leq i \leq 50$.}
 
 \textcolor{black}{\paragraph{ Section \ref{realex4}  } For all 10 datasets except for {\it Yacht Hydrodynamic} and {\it Year Prediction MSD}, 
 we set $\sigma_{n,2}^2 = 0.01$, $t_{HMC} = 25$, $\alpha = 0.1$, and $\epsilon_{t,i} = 5e-4$ 
 for all $t$ and $i$. For the SVRs in the first layer, we set $\epsilon = 0.01$. We used $\frac{1}{9}$ of the training data as the validation set and chose $\tilde{C}_{n,1}^{(t)} \in {1, 2, 5, 10, 20}$ with the smallest MSE on the validation set. For the dataset {\it Yacht Hydrodynamic}, we set $\sigma_{n,2}^2 = 0.0001$, $\alpha = 0.1$, $\epsilon_{t,i} = 5e-6$ and $\tilde{C}_{n,1}^{(t)}  = 200$. For the dataset {\it Year Prediction MSD}, we set $\sigma_{n,2}^2 = 0.02$, $\alpha = 0.1$, $\epsilon_{t,i} = 1e-3$ and $\tilde{C}_{n,1}^{(t)}  = 1$. Similar to the CoverType dataset, when some categorical features exist in the dataset, the default choice $\gamma = \frac{1}{p \text{Var}(\bx)}$ in the RBF kernel does not work very well. Among the 10 datasets, we set $\gamma = 3$ for {\it Yacht Hydrodynamic}, $\gamma = 1$ for {\it Protein Structure}, and employed the default setting for the others. } 
 
 \textcolor{black}{The KNN model was trained in a similar setting as used for the probabilistic back-propagation method in \cite{hernandez2015probabilistic}: we used a one-hidden layer model with 50 hidden units, and trained the model using SGD with a constant learning rate of 0.0001 and a momentum decay factor of 0.9. As in \cite{hernandez2015probabilistic}, we ran SGD for 40 epochs with a mini-batch size of 1.}

\paragraph{Section \ref{uncertainty2num}} {\it Prediction Interval.} 
Network: $C_n=10$ for the SVR layer, $\sigma_{n,2}^2 = 0.001$ for the output layer;
 HMC imputation: $t_{HMC}=25$, $\alpha = 0.1$, and $\epsilon_{t,i}=5e-6$ 
 for all $t$ and $i$; parameter updating: for all $t$ and $i$, (i) SVR with  $\tilde{C}_{n,1}^{(t)}=10$ and $\varepsilon=0.05$, (ii) linear regression, OLS estimation.

\bibliographystyle{plainnat}
\bibliography{reference}

\begin{thebibliography}{67}
\providecommand{\natexlab}[1]{#1}
\providecommand{\url}[1]{\texttt{#1}}
\expandafter\ifx\csname urlstyle\endcsname\relax
  \providecommand{\doi}[1]{doi: #1}\else
  \providecommand{\doi}{doi: \begingroup \urlstyle{rm}\Url}\fi

\bibitem[Allen-Zhu et~al.(2019)Allen-Zhu, Li, and Song]{AllenZhu2019ACT}
Zeyuan Allen-Zhu, Yuanzhi Li, and Zhao Song.
\newblock A convergence theory for deep learning via over-parameterization.
\newblock In \emph{ICML}, 2019.

\bibitem[Balasundaram et~al.(2013)Balasundaram, Gupta, and Gupta]{BalaSVR2013}
Subramanian Balasundaram, Dr.~Deepak Gupta, and Kapil Gupta.
\newblock Lagrangian support vector regression via unconstrained convex
  minimization.
\newblock \emph{Neural networks}, 51C:\penalty0 67--79, 12 2013.

\bibitem[Bohn et~al.(2019)Bohn, Griebel, and Rieger]{Bohn2019}
Bastian Bohn, Michael Griebel, and Christian Rieger.
\newblock A representer theorem for deep kernel learning.
\newblock \emph{Journal of Machine Learning Research}, 20\penalty0
  (64):\penalty0 1--32, 2019.
\newblock URL \url{http://jmlr.org/papers/v20/17-621.html}.

\bibitem[Celeux and Diebolt(1985)]{CeleuxDiebolt1995}
G.~Celeux and J.~Diebolt.
\newblock The sem algorithm: a probabilistic teacher algorithm derived from the
  em algorithm for the mixture problem.
\newblock \emph{Computational Statistics Quarterly}, 2:\penalty0 73--82, 1985.

\bibitem[Cheng et~al.(2018)Cheng, Chatterji, Bartlett, and
  Jordan]{Cheng2018underdamped}
Xiang Cheng, Niladri~S. Chatterji, Peter~L. Bartlett, and Michael~I. Jordan.
\newblock Underdamped langevin mcmc: A non-asymptotic analysis.
\newblock In Sébastien Bubeck, Vianney Perchet, and Philippe Rigollet,
  editors, \emph{Proceedings of the 31st Conference On Learning Theory},
  volume~75 of \emph{Proceedings of Machine Learning Research}, pages 300--323.
  PMLR, 06--09 Jul 2018.

\bibitem[Cho and Saul(2009)]{ChoSaul2009}
Youngmin Cho and Lawrence Saul.
\newblock Kernel methods for deep learning.
\newblock In Y.~Bengio, D.~Schuurmans, J.~Lafferty, C.~Williams, and
  A.~Culotta, editors, \emph{Advances in Neural Information Processing
  Systems}, volume~22. Curran Associates, Inc., 2009.
\newblock URL
  \url{https://proceedings.neurips.cc/paper/2009/file/5751ec3e9a4feab575962e78e006250d-Paper.pdf}.

\bibitem[Christmann and Steinwart(2007)]{SVRconsistency2007}
Andreas Christmann and Ingo Steinwart.
\newblock Consistency and robustness of kernel-based regression in convex risk
  minimization.
\newblock \emph{Bernoulli}, 13\penalty0 (3):\penalty0 799--819, 2007.

\bibitem[{Das} et~al.(2015){Das}, {Bhaduri}, {Matthews}, and
  {Oza}]{scalableSVR2015}
K.~{Das}, K.~{Bhaduri}, B.~L. {Matthews}, and N.~C. {Oza}.
\newblock Large scale support vector regression for aviation safety.
\newblock In \emph{2015 IEEE International Conference on Big Data (Big Data)},
  pages 999--1006, 2015.

\bibitem[Dempster et~al.(1977)Dempster, Laird, and Rubin]{Dempster1977}
A.P. Dempster, N.~Laird, and D.B. Rubin.
\newblock Maximum likelihood from incomplete data via the em algorithm.
\newblock \emph{\JRSSB}, 39:\penalty0 1--38, 1977.

\bibitem[Du et~al.(2019)Du, Lee, Li, Wang, and Zhai]{Du2019GradientDF}
Simon~S. Du, Jason~D. Lee, Haochuan Li, Liwei Wang, and Xiyu Zhai.
\newblock Gradient descent finds global minima of deep neural networks.
\newblock In \emph{ICML}, 2019.

\bibitem[Duane et~al.(1987)Duane, Kennedy, Pendleton, and Roweth]{HMC1987}
Simon Duane, Anthony~D. Kennedy, Brian~J. Pendleton, and Duncan Roweth.
\newblock Hybrid monte carlo.
\newblock \emph{Physics Letters B}, 195\penalty0 (2):\penalty0 216--222, 1987.

\bibitem[Fan and Li(2001)]{FanL2001}
J.~Fan and R.~Li.
\newblock Variable selection via nonconcave penalized likelihood and its oracle
  properties.
\newblock \emph{\JASA}, 96:\penalty0 1348--1360, 2001.

\bibitem[Gal and Ghahramani(2016)]{MCdropout2016}
Yarin Gal and Zoubin Ghahramani.
\newblock Dropout as a {B}ayesian approximation: Representing model uncertainty
  in deep learning.
\newblock In \emph{Proceedings of the 33rd International Conference on
  International Conference on Machine Learning - Volume 48}, ICML'16, page
  1050–1059. JMLR.org, 2016.

\bibitem[Gao et~al.(2002)Gao, Gunn, and Harris]{GaoBayesSVM2002}
J.B. Gao, S.R. Gunn, and C.J. Harris.
\newblock A probabilistic framework for svm regression and error bar
  estimation.
\newblock \emph{Machine Learning}, 46\penalty0 (1):\penalty0 71--89, 2002.

\bibitem[Geman and Geman(1984)]{GemanG1984}
S.~Geman and D.~Geman.
\newblock Stochastic relaxation, gibbs distributions and the bayesian
  restoration of images.
\newblock \emph{IEEE Transactions on Pattern Analysis and Machine
  Intelligence}, 6:\penalty0 721--741, 1984.

\bibitem[{Gori} and {Tesi}(1992)]{GoriTesi1992}
M.~{Gori} and A.~{Tesi}.
\newblock On the problem of local minima in backpropagation.
\newblock \emph{IEEE Transactions on Pattern Analysis and Machine
  Intelligence}, 14\penalty0 (1):\penalty0 76--86, 1992.

\bibitem[Graves(2011)]{graves2011practical}
Alex Graves.
\newblock Practical variational inference for neural networks.
\newblock In \emph{Advances in neural information processing systems}, pages
  2348--2356. Citeseer, 2011.

\bibitem[G{\"u}lçehre et~al.(2016)G{\"u}lçehre, Moczulski, Denil, and
  Bengio]{Glehre2016NoisyAF}
C.~G{\"u}lçehre, M.~Moczulski, M.~Denil, and Y.~Bengio.
\newblock Noisy activation functions.
\newblock In \emph{ICML}, 2016.

\bibitem[Guo et~al.(2017)Guo, Pleiss, Sun, and Weinberger]{CalibrationDNN2017}
Chuan Guo, Geoff Pleiss, Yu~Sun, and Kilian~Q. Weinberger.
\newblock On calibration of modern neural networks.
\newblock In \emph{Proceedings of the 34th International Conference on Machine
  Learning - Volume 70}, ICML'17, page 1321–1330. JMLR.org, 2017.

\bibitem[Hammer and Gersmann(2003)]{Hammer2003}
B.~Hammer and K.~Gersmann.
\newblock A note on the universal approximation capability of support vector
  machines.
\newblock \emph{Neural Processing Letters}, 17:\penalty0 43--53, 2003.

\bibitem[Hern{\'a}ndez-Lobato and Adams(2015)]{hernandez2015probabilistic}
Jos{\'e}~Miguel Hern{\'a}ndez-Lobato and Ryan Adams.
\newblock Probabilistic backpropagation for scalable learning of {B}ayesian
  neural networks.
\newblock In \emph{International Conference on Machine Learning}, pages
  1861--1869. PMLR, 2015.

\bibitem[Hinton(2007)]{Hinton2007}
G.E. Hinton.
\newblock Learning multiple layers of representation.
\newblock \emph{Trends in Cognitive Sciences}, 11\penalty0 (10):\penalty0
  428--434, 2007.

\bibitem[Ioffe and Szegedy(2015)]{batchnormalization2015}
Sergey Ioffe and Christian Szegedy.
\newblock Batch normalization: Accelerating deep network training by reducing
  internal covariate shift.
\newblock In \emph{Proceedings of the 32nd International Conference on
  International Conference on Machine Learning - Volume 37}, ICML'15, page
  448–456. JMLR.org, 2015.

\bibitem[Kingma and Ba(2014)]{kingma2014adam}
Diederik~P Kingma and Jimmy Ba.
\newblock Adam: A method for stochastic optimization.
\newblock \emph{arXiv preprint arXiv:1412.6980}, 2014.

\bibitem[Kuhn and Johnson(2013)]{KuhnJ2013}
Max Kuhn and KjeII Johnson.
\newblock \emph{Applied Predictive Modeling}.
\newblock Springer-Verlag, New York, 2013.

\bibitem[Lakshminarayanan et~al.(2017)Lakshminarayanan, Pritzel, and
  Blundell]{Deepensemble2017}
Balaji Lakshminarayanan, Alexander Pritzel, and Charles Blundell.
\newblock Simple and scalable predictive uncertainty estimation using deep
  ensembles.
\newblock In \emph{Proceedings of the 31st International Conference on Neural
  Information Processing Systems}, NIPS'17, page 6405–6416, Red Hook, NY,
  USA, 2017. Curran Associates Inc.
\newblock ISBN 9781510860964.

\bibitem[Laskov et~al.(2006)Laskov, Gehl, Kr\"{u}ger, and
  M\"{u}ller]{IncrementalSVM2006}
Pavel Laskov, Christian Gehl, Stefan Kr\"{u}ger, and Klaus-Robert M\"{u}ller.
\newblock Incremental support vector learning: Analysis, implementation and
  applications.
\newblock \emph{J. Mach. Learn. Res.}, 7:\penalty0 1909–1936, December 2006.
\newblock ISSN 1532-4435.

\bibitem[LeCun et~al.(1998)LeCun, Bottou, Bengio, and
  Haffner]{lecun1998gradient}
Yann LeCun, L{\'e}on Bottou, Yoshua Bengio, and Patrick Haffner.
\newblock Gradient-based learning applied to document recognition.
\newblock \emph{Proceedings of the IEEE}, 86\penalty0 (11):\penalty0
  2278--2324, 1998.

\bibitem[Liang et~al.(2018)Liang, Jia, Xue, Li, and Luo]{Liang2018missing}
F.~Liang, B.~Jia, J.~Xue, Q.~Li, and Y.~Luo.
\newblock An imputation-regularized optimization algorithm for high-dimensional
  missing data problems and beyond.
\newblock \emph{\JRSSB}, 80\penalty0 (5):\penalty0 899--926, 2018.

\bibitem[Loh(2017)]{Loh2017Ann}
P.-L. Loh.
\newblock Statistical consistency and asymptotic normality for high-dimensional
  robust m-estimators.
\newblock \emph{\ANNALS}, 45\penalty0 (2):\penalty0 866--896, 2017.

\bibitem[Loh et~al.(2017)Loh, Wainwright, et~al.]{loh2017support}
Po-Ling Loh, Martin~J Wainwright, et~al.
\newblock Support recovery without incoherence: A case for nonconvex
  regularization.
\newblock \emph{The Annals of Statistics}, 45\penalty0 (6):\penalty0
  2455--2482, 2017.

\bibitem[Ma et~al.(2003)Ma, Theiler, and Perkins]{OnlineSVR2003}
Junshui Ma, James Theiler, and Simon Perkins.
\newblock Accurate on-line support vector regression.
\newblock \emph{Neural Computation}, 15\penalty0 (11):\penalty0 2683–2703,
  2003.

\bibitem[Mairal et~al.(2014)Mairal, Koniusz, Harchaoui, and
  Schmid]{mairal2014convolutional}
Julien Mairal, Piotr Koniusz, Zaid Harchaoui, and Cordelia Schmid.
\newblock Convolutional kernel networks.
\newblock \emph{arXiv preprint arXiv:1406.3332}, 2014.

\bibitem[Micchelli et~al.(2006)Micchelli, Xu, and Zhang]{UnivKernel2006}
C.A. Micchelli, Y.~Xu, and H.~Zhang.
\newblock Universal kernels.
\newblock \emph{Journal of Machine Learning Research}, 7:\penalty0 2651--2667,
  2006.

\bibitem[Nakkiran et~al.(2020)Nakkiran, Kaplun, Bansal, Yang, Barak, and
  Sutskever]{Doubledescent2020}
Preetum Nakkiran, Gal Kaplun, Yamini Bansal, Tristan Yang, Boaz Barak, and Ilya
  Sutskever.
\newblock Deep double descent: Where bigger models and more data hurt.
\newblock In \emph{International Conference on Learning Representations}, 2020.
\newblock URL \url{https://openreview.net/forum?id=B1g5sA4twr}.

\bibitem[Neal(2011)]{HMC-Neal2011}
Radford~M. Neal.
\newblock Mcmc using hamiltonian dynamics.
\newblock In Andrew Gelman, Galin~L. Jones, and Xiao-Li Meng, editors,
  \emph{Handbook of Markov Chain Monte Carlo}, chapter~5, pages 113--162.
  Chapman and Hall/CRC, 2011.

\bibitem[Neelakantan et~al.(2017)Neelakantan, Vilnis, Le, Sutskever, Kaiser,
  Kurach, and Martens]{Neelakantan2017AddingGN}
Arvind Neelakantan, Luke Vilnis, Quoc~V. Le, Ilya Sutskever, Lukasz Kaiser,
  Karol Kurach, and James Martens.
\newblock Adding gradient noise improves learning for very deep networks.
\newblock \emph{ArXiv}, abs/1511.06807, 2017.

\bibitem[Nguyen and Hein(2017)]{Nguyen2017TheLS}
Quynh Nguyen and Matthias Hein.
\newblock The loss surface of deep and wide neural networks.
\newblock In \emph{ICML}, 2017.

\bibitem[Nielsen(2000)]{Nielsen2000}
S.F. Nielsen.
\newblock The stochastic em algorithm: Estimation and asymptotic results.
\newblock \emph{Bernoulli}, 6:\penalty0 457--489, 2000.

\bibitem[Noh et~al.(2017)Noh, You, Mun, and Han]{Noh2017RegularizingDN}
Hyeonwoo Noh, Tackgeun You, Jonghwan Mun, and Bohyung Han.
\newblock Regularizing deep neural networks by noise: Its interpretation and
  optimization.
\newblock \emph{ArXiv}, abs/1710.05179, 2017.

\bibitem[{Park} and {Sandberg}(1991)]{RBFUniv1991}
J.~{Park} and I.~W. {Sandberg}.
\newblock Universal approximation using radial-basis-function networks.
\newblock \emph{Neural Computation}, 3\penalty0 (2):\penalty0 246--257, 1991.

\bibitem[Pedregosa et~al.(2011)Pedregosa, Varoquaux, Gramfort, Michel, Thirion,
  Grisel, Blondel, Prettenhofer, Weiss, Dubourg, Vanderplas, Passos,
  Cournapeau, Brucher, Perrot, and {{\'E}}douard Duchesnay]{sklearn2011}
Fabian Pedregosa, Ga{{\"e}}l Varoquaux, Alexandre Gramfort, Vincent Michel,
  Bertrand Thirion, Olivier Grisel, Mathieu Blondel, Peter Prettenhofer, Ron
  Weiss, Vincent Dubourg, Jake Vanderplas, Alexandre Passos, David Cournapeau,
  Matthieu Brucher, Matthieu Perrot, and {{\'E}}douard Duchesnay.
\newblock Scikit-learn: Machine learning in python.
\newblock \emph{Journal of Machine Learning Research}, 12\penalty0
  (85):\penalty0 2825--2830, 2011.
\newblock URL \url{http://jmlr.org/papers/v12/pedregosa11a.html}.

\bibitem[Portnoy(1989)]{Portnoy1988}
S.~Portnoy.
\newblock Asymptotic behavior of likelihood methods for exponential families
  when the number of parameters tend to infinity.
\newblock \emph{\ANNALS}, 16\penalty0 (1):\penalty0 356--366, 1989.

\bibitem[Rebai et~al.(2016)Rebai, BenAyed, and Mahdi]{Rebai2016}
I.~Rebai, Y.~BenAyed, and W.~Mahdi.
\newblock Deep multilayer multiple kernel learning.
\newblock \emph{Neural Computing \& Applications}, 27:\penalty0 2305--2314,
  2016.

\bibitem[Rossky et~al.(1978)Rossky, Doll, and Friedman]{Rossky1978BrownianDA}
Peter~J Rossky, Jimmie~D. Doll, and Harold~L. Friedman.
\newblock Brownian dynamics as smart monte carlo simulation.
\newblock \emph{Journal of Chemical Physics}, 69:\penalty0 4628--4633, 1978.

\bibitem[{Ruta} et~al.(2019){Ruta}, {Cen}, and {Vu}]{GreedySVR2019}
D.~{Ruta}, L.~{Cen}, and Q.~H. {Vu}.
\newblock Greedy incremental support vector regression.
\newblock In \emph{2019 Federated Conference on Computer Science and
  Information Systems (FedCSIS)}, pages 7--9, 2019.

\bibitem[Salakhutdinov and Hinton(2009)]{SHinton2009}
R.~Salakhutdinov and G.~Hinton.
\newblock Deep boltzmann machines.
\newblock In \emph{Proceedings of the International Conference on Artificial
  Intelligence and Statistics}, pages 448--455, 2009.

\bibitem[Sch{\"o}lkopf et~al.(2001)Sch{\"o}lkopf, Herbrich, and
  Smola]{Scholkopf2001}
Bernhard Sch{\"o}lkopf, Ralf Herbrich, and Alex~J. Smola.
\newblock A generalized representer theorem.
\newblock In David Helmbold and Bob Williamson, editors, \emph{Computational
  Learning Theory}, pages 416--426, Berlin, Heidelberg, 2001. Springer Berlin
  Heidelberg.

\bibitem[Smola and Sch{\"o}lkopf(2004)]{smola2004tutorial}
Alex~J Smola and Bernhard Sch{\"o}lkopf.
\newblock A tutorial on support vector regression.
\newblock \emph{Statistics and computing}, 14\penalty0 (3):\penalty0 199--222,
  2004.

\bibitem[Srivastava et~al.(2014)Srivastava, Hinton, Krizhevsky, Sutskever, and
  Salakhutdinov]{srivastava2014dropout}
Nitish Srivastava, Geoffrey Hinton, Alex Krizhevsky, Ilya Sutskever, and Ruslan
  Salakhutdinov.
\newblock Dropout: a simple way to prevent neural networks from overfitting.
\newblock \emph{The journal of machine learning research}, 15\penalty0
  (1):\penalty0 1929--1958, 2014.

\bibitem[Steinwart(2002)]{Steinwartkernel2002}
Ingo Steinwart.
\newblock On the influence of the kernel on the consistency of support vector
  machines.
\newblock \emph{Journal of Machine Learning Research}, 2\penalty0 (1):\penalty0
  67--93, 2002.

\bibitem[Strobl and Visweswaran(2013)]{StroblV2013}
Eric~V. Strobl and Shyam Visweswaran.
\newblock Deep multiple kernel learning.
\newblock In \emph{2013 12th International Conference on Machine Learning and
  Applications}, volume~1, pages 414--417, 2013.
\newblock \doi{10.1109/ICMLA.2013.84}.

\bibitem[Sun et~al.(2021)Sun, Song, and Liang]{SunSLiang2021}
Y.~Sun, Q.~Song, and F.~Liang.
\newblock Consistent sparse deep learning: Theory and computation.
\newblock \emph{Journal of the American Statistical Association}, page in
  press, 2021.

\bibitem[Tibshirani(1996)]{Tibshirani1996}
R.~Tibshirani.
\newblock Regression shrinkage and selection via the lasso.
\newblock \emph{\JRSSB}, 58:\penalty0 267--288, 1996.

\bibitem[Vapnik and Chervonenkis(1964)]{VapnikC1964}
V.~Vapnik and A.~Chervonenkis.
\newblock A note on one class of perceptrons.
\newblock \emph{Automation and Remote Control}, 25:\penalty0 774--780, 1964.

\bibitem[Vapnik and Lerner(1963)]{VapnikL1963}
V.~Vapnik and A.~Lerner.
\newblock Pattern recognition using generalized portrait method.
\newblock \emph{Automation and Remote Control}, 24:\penalty0 774--780, 1963.

\bibitem[Vapnik(2000)]{Vapnik2000}
V.N. Vapnik.
\newblock \emph{The Nature of Statistical Learning Theory (2nd ed.)}.
\newblock Springer, New York, 2000.

\bibitem[Wahba(1990)]{Wahba1990}
Grace Wahba.
\newblock \emph{Splines models for observational data}.
\newblock SIAM, Philadelphia, 1990.

\bibitem[Wang and Rockov{\'a}(2020)]{Wang2020UncertaintyQF}
Yuexi Wang and V.~Rockov{\'a}.
\newblock Uncertainty quantification for sparse deep learning.
\newblock In \emph{AISTATS}, 2020.

\bibitem[Wen et~al.(2018)Wen, Shi, Li, He, and Chen]{wenthundersvm18}
Zeyi Wen, Jiashuai Shi, Qinbin Li, Bingsheng He, and Jian Chen.
\newblock {ThunderSVM}: A fast {SVM} library on {GPUs} and {CPUs}.
\newblock \emph{Journal of Machine Learning Research}, 19:\penalty0 797--801,
  2018.

\bibitem[Wilson et~al.(2016)Wilson, Hu, Salakhutdinov, and
  Xing]{wilsonXing2016}
{Andrew Gordon} Wilson, Zhiting Hu, Ruslan Salakhutdinov, and {Eric P.} Xing.
\newblock Deep kernel learning.
\newblock In \emph{19th International Conference on Artificial Intelligence and
  Statistics, AISTATS 2016}, pages 370--378, 2016.

\bibitem[{Xu} et~al.(2010){Xu}, {Wang}, and {Wang}]{XuBoundarySVR2010}
H.~{Xu}, R.~{Wang}, and K.~{Wang}.
\newblock A new svr incremental algorithm based on boundary vector.
\newblock In \emph{2010 International Conference on Computational Intelligence
  and Software Engineering}, pages 1--4, 2010.

\bibitem[You et~al.(2018)You, Ye, Li, and Wang]{You2018AdversarialNL}
Zhonghui You, Jinmian Ye, Kunming Li, and Ping Wang.
\newblock Adversarial noise layer: Regularize neural network by adding noise.
\newblock In \emph{2019 IEEE International Conference on Image Processing
  (ICIP)}, pages 909--913, 2018.

\bibitem[Zhang(2010)]{Zhang2010}
C.-H. Zhang.
\newblock Nearly unbiased variable selection under minimax concave penalty.
\newblock \emph{\ANNALS}, 38:\penalty0 894--942, 2010.

\bibitem[Zhuang et~al.(2011)Zhuang, Tsang, and Hoi]{zhuang2011}
Jinfeng Zhuang, Ivor~W. Tsang, and Steven~C.H. Hoi.
\newblock Two-layer multiple kernel learning.
\newblock In Geoffrey Gordon, David Dunson, and Miroslav Dudík, editors,
  \emph{Proceedings of the Fourteenth International Conference on Artificial
  Intelligence and Statistics}, volume~15 of \emph{Proceedings of Machine
  Learning Research}, pages 909--917, Fort Lauderdale, FL, USA, 11--13 Apr
  2011. PMLR.
\newblock URL \url{http://proceedings.mlr.press/v15/zhuang11a.html}.

\bibitem[Zou and Gu(2019)]{ZouGu2019}
Difan Zou and Quanquan Gu.
\newblock An improved analysis of training over-parameterized deep neural
  networks.
\newblock In \emph{NuerIPS}, 2019.

\bibitem[Zou et~al.(2020)Zou, Cao, Zhou, and Gu]{Zou2020GradientDO}
Difan Zou, Yuan Cao, Dongruo Zhou, and Quanquan Gu.
\newblock Gradient descent optimizes over-parameterized deep relu networks.
\newblock \emph{Machine Learning}, 109:\penalty0 467 -- 492, 2020.

\end{thebibliography}

\end{document}